\PassOptionsToPackage{table}{xcolor}
\documentclass{article} 
\usepackage{iclr2027_conference,times}

\usepackage{amsthm}
\usepackage{amsmath}
\usepackage{amssymb}
\usepackage{graphicx}
\usepackage{mathtools}
\usepackage{enumerate}
\usepackage{enumitem}
\usepackage{footnote}
\usepackage{float}
\usepackage{xspace}
\usepackage{multirow}
\usepackage{nicefrac}
\usepackage{wrapfig}
\usepackage{framed}
\usepackage{url}
\usepackage{units}
\usepackage[colorlinks=true, linkcolor=black, citecolor=blue, urlcolor=magenta]{hyperref}
\usepackage{blkarray}
\PassOptionsToPackage{round}{natbib}

\usepackage{tikz}
\usetikzlibrary{arrows,chains,matrix,positioning,scopes}

\let\hat\widehat
\let\tilde\widetilde

\newtheorem{lemma}{Lemma}

\newtheorem{theorem}{Theorem}

\newtheorem{remark}{Remark}

\newtheorem{assumption}{Assumption}

\DeclareMathOperator*{\argmin}{argmin}

\DeclareFontFamily{OMX}{MnSymbolE}{}
\DeclareFontShape{OMX}{MnSymbolE}{m}{n}{
    <-6>  MnSymbolE5
   <6-7>  MnSymbolE6
   <7-8>  MnSymbolE7
   <8-9>  MnSymbolE8
   <9-10> MnSymbolE9
  <10-12> MnSymbolE10
  <12->   MnSymbolE12}{}
\DeclareSymbolFont{mnlargesymbols}{OMX}{MnSymbolE}{m}{n}
\SetSymbolFont{mnlargesymbols}{bold}{OMX}{MnSymbolE}{b}{n}
\DeclareMathDelimiter{\llangle}{\mathopen}{mnlargesymbols}{'164}{mnlargesymbols}{'164}
\DeclareMathDelimiter{\rrangle}{\mathclose}{mnlargesymbols}{'171}{mnlargesymbols}{'171}

\usepackage{prettyref}

\newcommand{\savehyperref}[2]{\texorpdfstring{\hyperref[#1]{#2}}{#2}}
\newrefformat{eq}{\savehyperref{#1}{\textup{(\ref*{#1})}}}
\newrefformat{eqn}{\savehyperref{#1}{Equation~\eqref{#1}}}
\newrefformat{lem}{\savehyperref{#1}{Lemma~\ref*{#1}}}
\newrefformat{def}{\savehyperref{#1}{Definition~\ref*{#1}}}
\newrefformat{line}{\savehyperref{#1}{line~\ref*{#1}}}
\newrefformat{thm}{\savehyperref{#1}{Theorem~\ref*{#1}}}
\newrefformat{corr}{\savehyperref{#1}{Corollary~\ref*{#1}}}
\newrefformat{cor}{\savehyperref{#1}{Corollary~\ref*{#1}}}
\newrefformat{sec}{\savehyperref{#1}{Section~\ref*{#1}}}
\newrefformat{app}{\savehyperref{#1}{Appendix~\ref*{#1}}}
\newrefformat{ap}{\savehyperref{#1}{Appendix~\ref*{#1}}}
\newrefformat{assum}{\savehyperref{#1}{Assumption~\ref*{#1}}}
\newrefformat{as}{\savehyperref{#1}{Assumption~\ref*{#1}}}
\newrefformat{ex}{\savehyperref{#1}{Example~\ref*{#1}}}
\newrefformat{fig}{\savehyperref{#1}{Figure~\ref*{#1}}}
\newrefformat{alg}{\savehyperref{#1}{Algorithm~\ref*{#1}}}
\newrefformat{rem}{\savehyperref{#1}{Remark~\ref*{#1}}}
\newrefformat{conj}{\savehyperref{#1}{Conjecture~\ref*{#1}}}
\newrefformat{prop}{\savehyperref{#1}{Proposition~\ref*{#1}}}
\newrefformat{proto}{\savehyperref{#1}{Protocol~\ref*{#1}}}
\newrefformat{prob}{\savehyperref{#1}{Problem~\ref*{#1}}}
\newrefformat{claim}{\savehyperref{#1}{Claim~\ref*{#1}}}
\newrefformat{que}{\savehyperref{#1}{Question~\ref*{#1}}}
\usepackage{bbm}
\usepackage{graphicx}
\usepackage{subcaption}
\usepackage{booktabs}
\usepackage[utf8]{inputenc} 
\usepackage[T1]{fontenc}    
\usepackage{algorithm}     
\usepackage{algpseudocode} 
\usepackage{hyperref}       
\usepackage{url}            
\usepackage{amsfonts}       
\usepackage{nicefrac}       
\usepackage{microtype}      
\usepackage[table]{xcolor}         
\usepackage{multirow}
\usepackage{algorithm}
\usepackage{algpseudocode}
\usepackage{subcaption}
\newtheorem{proposition}{Proposition}
\usepackage{xcolor}
\usepackage{bm}
\usepackage{amsmath,amssymb,amsthm,mathtools}
\usepackage{bm}

\usepackage{hyperref}
\usepackage{url}
\usepackage{algorithm}
\usepackage{algpseudocode}
\usepackage{subcaption}

\definecolor{mycyan}{RGB}{0,204,204}
\definecolor{mygreen}       {rgb} {0.10,0.50,0.10}
\definecolor{mylightgreen}  {RGB} {50,168,82}
\definecolor{lightred}      {rgb} {1.00,0.80,0.80}
\definecolor{lightblue}     {rgb} {0.80,0.80,1.00}
\definecolor{mediumred}     {rgb} {1.00,0.60,0.60}
\definecolor{mediumblue}    {rgb} {0.40,0.50,0.70}
\definecolor{pennblue}{cmyk}{1,0.65,0,0.30}
\definecolor{pennred}{cmyk}{0,1,0.65,0.34}
\definecolor{myred}{RGB}{199,107,102}

\newcommand{\reals}{\mathbb{R}}

\title{Graph Convolutional Attention: A Spectral Perspective on Graph Denoising and Diffusion}

\author{%
\mdseries
\begin{tabular}[t]{@{}>{\centering\arraybackslash}p{0.47\textwidth}>{\centering\arraybackslash}p{0.47\textwidth}@{}}
\rule{0pt}{4ex}\textbf{Shervin Khalafi}\thanks{Equal contribution.} &
\textbf{Igor Krawczuk}\footnotemark[1] \\
University of Pennsylvania & {\small\texttt{igor@krawczuk.eu}} \\
{\small\texttt{shervink@seas.upenn.edu}} &
 \\[2ex]
\textbf{Sergio Rozada} & \textbf{Charilaos Kanatsoulis} \\
King Juan Carlos University & Stanford University \\
{\small\texttt{sergio.rozada@urjc.es}} &
{\small\texttt{charilaos@cs.stanford.edu}} \\[2ex]
\textbf{Antonio G. Marques} & \textbf{Alejandro Ribeiro} \\
King Juan Carlos University & University of Pennsylvania \\
{\small\texttt{antonio.garcia.marques@urjc.es}} &
{\small\texttt{aribeiro@seas.upenn.edu}}
\end{tabular}}

\iclrfinalcopy 
\begin{document}

\maketitle

\lhead{}
\renewcommand{\headrulewidth}{0pt}

\begin{abstract}
Denoising graphs is a fundamental problem in graph learning and the core operation of graph diffusion models. Attention-based architectures such as graph transformers have shown promise for this task, yet the mechanisms through which attention denoises graphs remain poorly understood. We show that, under a graph denoising objective, standard attention scores learn an average spectral denoising filter across the training distribution and are suboptimal when the graph spectra vary. To address this limitation, we introduce spectral attention scores, which are attention scores that depend on the input graph spectrum. We characterize the improvement over standard attention scores in terms of the spectral diversity of the graph distribution.
We then propose graph convolutional attention (GCA), a practical realization of spectral attention that uses graph-filtered queries and keys. For asymptotically large stochastic block models, we show that GCA matches the optimal denoising loss of spectral attention under suitable conditions. We also show that softmax can further reduce denoising error by approximately projecting noisy eigenvectors onto the clean eigenspace. Across synthetic and real datasets, replacing standard query--key projections with GCA improves graph denoising and diffusion, with denoising gains correlated with spectral diversity. Furthermore, replacing the DiGress architecture with a GCA-based architecture within the DiGress generative framework yields comparable generation quality with significantly faster inference.\footnote{Our code can be found here: \url{https://github.com/shervinkhalafi/graph_conv_att}}
\end{abstract}
\section{Introduction}\label{sec: introduction}

Denoising is a fundamental problem in machine learning~\citep{milanfar2025denoising}, both in its own right and as the central learning primitive behind denoising diffusion models~\citep{song2020denoising, song2020score}. 
Graph denoising introduces a unique challenge: the object to be denoised is not merely a signal on a known domain, but the domain itself. 
In image and audio denoising, the underlying grid or temporal axis is fixed, so models can exploit geometry through convolutions~\citep{milanfar2012tour}. 
In graphs, by contrast, the geometry and spectral content are corrupted and must themselves be recovered.
Classical approaches address graph denoising through regularized optimization \citep{wu2017generalized, segarra2017network}, but the parametric perspective has received comparatively little attention. Interest has grown recently, driven by graph diffusion models \citep{jo2022score, vignac2022digress}, where the dominant approach relies on graph transformers that take as input the eigenvectors of the graph, and optionally, additional structural features. Yet, a theoretical understanding of how such parametric models learn to denoise graphs remains limited. In this work, we aim to provide insight into the mechanisms underlying parametric graph denoising through theoretical analysis.

We proceed in four steps. First, we (i) analyze the attention-score matrix, the outer product of linear query-key projections, which underlies standard attention-based graph denoisers. We show that it learns an average of the spectral properties across the graph distribution and is therefore suboptimal when the spectra vary. 
To address this limitation, we (ii) introduce spectral attention, a general class whose score matrices depend arbitrarily on the input eigenvalues. We show that its minimum denoising loss is lower than that achieved using standard attention scores by an amount determined by the spectral diversity of the graph distribution, i.e., how diverse the spectra are across the distribution. Spectral attention, however, is an abstract class. 
To turn the insight into something practical, we (iii) introduce graph convolutional attention (GCA), in which the attention-score matrix becomes a pointwise function of the eigenvalues. We show that GCA scores are approximated arbitrarily well by products of graph filters, and that, for asymptotically large stochastic block models (SBMs), GCA matches the optimal loss of spectral attention under suitable conditions. Finally, we (iv) study the softmax applied to the attention-score matrix. Assuming that the clean eigenvalues have been recovered, we show that softmax can further denoise the graph by approximately projecting the noisy eigenvectors onto the clean eigenspace.

Experimentally, across real and synthetic datasets, a graph transformer using GCA in place of the standard query-key projections outperforms the standard graph transformer at denoising, and the per-dataset improvement correlates strongly with a surrogate for the spectral diversity metric, validating the theory. For diffusion models, GCA reduces training loss and, on several datasets, sample quality both qualitatively and under standard metrics. Finally, replacing the DiGress architecture with a simpler GCA-based architecture within the same generative framework achieves comparable graph generation quality while reducing inference cost and enabling faster graph generation.

We summarize our contributions as follows:

\begin{enumerate}
    \item[\textbf{C1}] \textbf{Theoretical analysis of attention-based denoising.}
    We show that standard attention scores learn an average of the spectral properties across the graph distribution and are suboptimal when those properties vary. We introduce spectral attention and characterize its improvement in minimum denoising loss through the spectral diversity of the distribution.

    \item[\textbf{C2}] \textbf{Graph convolutional attention (GCA).}
    We propose GCA, an attention mechanism whose score matrix is approximated arbitrarily well by graph convolutional filters. For asymptotically large stochastic block models (SBMs), we show that, under suitable conditions, GCA matches the optimal loss of spectral attention.

    \item[\textbf{C3}] \textbf{Role of softmax.}
    In an idealized setting where the clean eigenvalues are known, we show that softmax further improves denoising by approximately projecting the noisy graph onto the clean eigenspace.

    \item[\textbf{C4}] \textbf{Empirical validation.}
    Across real and synthetic datasets, replacing standard query--key projections with GCA improves graph denoising. The gains across datasets correlate with a surrogate for spectral diversity and extend to graph diffusion, where GCA reduces training loss and improves sample quality on several datasets.

    \item[\textbf{C5}] \textbf{Implications for modern diffusion models.}
    Replacing the DiGress architecture with a simpler GCA-based architecture within the same generative framework yields comparable graph generation quality across the evaluated datasets, while reducing inference time.
\end{enumerate}

\subsection{Literature review}

\paragraph{Graph denoising.}
Existing graph denoising methods can be broadly divided into optimization-based and learned approaches. 
The former recover a clean graph by solving an inverse problem regularized with structural priors such as edge-error models \citep{wu2017generalized,josephs2021network}, spectral consistency \citep{segarra2017network,rey2023robust, tenorio2024blind}, or low-rank, sparse, and smooth graph structure \citep{zhou2024graph,wang2024unsupervised, lyu2024learning}. The latter replace instance-wise optimization with a parametric denoiser, either as a preprocessing tool or as a jointly trained module for downstream GNNs \citep{fatemi2021slaps,dai2022towards,yang2024you}, or as the core component of diffusion and flow-based generative models that iteratively remove graph corruption \citep{li2021maskgvae,jo2022score,vignac2022digress}. Overall, prior work spans from hand-crafted regularization to fully learned graph restoration, but most methods are not designed to explicitly exploit spectral structure within the denoising mechanism itself.

\paragraph{Spectral attention.}
A parallel line of work incorporates graph spectral information into Transformer architectures. Early methods use Laplacian eigenvectors or eigenvalues as spectral features or positional encodings for attention-based models \citep{he2020spectral,kreuzer2021rethinking,dwivedi2021generalizationtransformernetworksgraphs,pengmei2024technical}, while later approaches address eigenvector ambiguities through invariant spectral encoders such as SignNet and BasisNet \citep{lim2022sign}. More recent models integrate spectral structure directly into the attention mechanism or the learned operator. GIST reweights attention using inner products of random filtered spectral features \citep{rigotti2026gist}, whereas Specformer, PolyFormer, and Eigenformer use Transformer blocks to parametrize spectral filters using attention \citep{bo2023specformer,ma2024polyformer,garg2024graph}. These works show that spectral information can be embedded into attention in several ways, but they are mainly aimed at representation learning and prediction rather than graph denoising.

\section{Denoising Graphs with a Linear Model}\label{sec: linear theory}

\paragraph{Problem Setup.}
We begin by introducing the graph denoising problem as it is fundamental to graph diffusion. In particular, Sections~\ref{sec: linear theory} and \ref{sec: nonlinear theory} analyze how attention-based neural networks perform denoising from a spectral perspective. We show in this section that the attention scores learn to denoise the eigenvalues of a noisy graph, while, in Section~\ref{sec: nonlinear theory}, we show that the softmax nonlinearity helps remove noise from the eigenvectors. Regarding the setup, we consider unweighted, undirected graphs $\mathcal{G}$ sampled from a distribution over graphs which we denote by $\mathcal{D}_{\mathcal{G}}$.
For a given graph $\mathcal{G}$, we denote by $A \in \mathbb{R}^{n \times n}$ the symmetric adjacency matrix. We denote the noisy adjacency by $\tilde{A} = A + \mathcal{E}$, where $\mathcal{E} \in \mathbb{R}^{n \times n}$ is a symmetric noise matrix sampled from a noise distribution. We note that depending on the noising method this distribution could depend on the input graph (flipping edge values) or it could be independent (Gaussian noise). We write the eigendecomposition of the adjacency as $A = U \Lambda U^\top$, and that of the noisy adjacency as $\tilde{A} = \tilde{U} \tilde{\Lambda} \tilde{U}^\top$.

\subsection{Eigenvalue Information Improves Denoising.}

Consider an arbitrary denoising function $g(\cdot):\mathbb{R}^{n \times n}\rightarrow \mathbb{R}^{n \times n}$ that takes the noisy adjacency as input and returns a prediction of the denoised graph. We define the following loss functional characterizing the denoising capability of $g$ over different graphs and noise samples from the distribution:
\begin{equation}
    \label{eq::problem}
    \ell(g) = \,\,\mathbb{E}_{A,\,\mathcal{E}} \left\| g(\tilde{A}) - A \right\|_F^2.
\end{equation}

A common candidate for implementing the denoising map is attention~\citep{vignac2022digress, qin2024defog}. For a sequence of $d$-dimensional input tokens $X \in \reals^{n \times d}$, standard attention computes
\begin{equation}
   a(X) = \sigma(X W_Q W_K^\top X^\top)  X W_V,
\end{equation}
where $\sigma$ is the row-wise softmax activation function, and $W_Q \in \reals^{d \times d_k}$, $W_K \in \reals^{d \times d_k}$, and $W_V \in \reals^{d \times d_v}$ are learnable query, key, and value matrices, respectively.
In graph denoising, attention-based estimators of the adjacency can be built as an outer product of representations~\citep{vignac2022digress}
\begin{equation}
    \label{eq::graph_based_att}
    g(\tilde A)=a(\tilde U)a(\tilde U)^\top = \sigma(\underbrace{\tilde{U} W_Q W_K^\top \tilde{U}^\top}_{f(\tilde{A};W_Q, W_K)}) \cdot  \tilde{U} W_V W_V^\top \tilde{U}^\top \cdot \sigma(\tilde{U} W_Q W_K^\top \tilde{U}^\top)^\top.
\end{equation}
In this view, attention acts on $\tilde U$ to recover a spectral representation of the clean graph $A$. We analyze this mechanism in two stages. First, we study the attention-score matrix
$f(\tilde A;W_Q,W_K)=\tilde U W_QW_K^\top\tilde U^\top$.
This isolates the effect of the scores from the softmax and the value projections: we ask how well they can approximate $A$ by learning its representation in the noisy eigenbasis $\tilde U$. Second, we study in the next section how the softmax and the outer product of the resulting representations can further reduce the error associated with the noisy eigenvectors. Together, these stages distinguish the spectral information learned by the attention scores from the additional denoising provided by the rest of the attention-based estimator.
To analyze the first stage, we consider three classes of simplified denoising maps associated with the attention scores:

\begin{enumerate}
    \item \textbf{Standard Attention Scores:} We consider $f$ from the standard function class
    \begin{equation}
        \mathcal{F}_{\text{A}} = \left\{ f(\tilde{A}; W_Q, W_K) = \tilde{U} W_Q W_K^\top \tilde{U}^\top |\,\, W_Q \in \mathbb{R}^{n \times d_k}, W_K \in \mathbb{R}^{n \times d_k}  \right\}.
    \end{equation}

    \item \textbf{Spectral Attention Scores:} The broader class of functions $f\in \mathcal{F}_{\text{SA}}$ where
    attention depends on the eigenvalues of the noisy input graph via any arbitrary function $h(\cdot)$
    \begin{equation}
        \mathcal{F}_{\text{SA}} = \left\{ f(\tilde{A}; h(\cdot)) = \tilde{U} h(\tilde{\Lambda}) \tilde{U}^\top |\,\, h(\cdot):\mathbb{R}^{n} \rightarrow \mathbb{R}^{n \times n} \right\}.
    \end{equation}

    \item \textbf{Graph Convolutional Attention Scores:} We consider the class of score functions $f \in \mathcal{F}_{\text{GCA}}$ where attention is
    replaced by a function $\eta(\cdot)$ acting point-wise on the eigenvalues
    \begin{equation}\label{eq: GCA definition}
        \mathcal{F}_{\text{GCA}} = \left\{ f(\tilde{A}; \eta(\cdot)) = \tilde{U} \eta(\tilde{\Lambda}) \tilde{U}^\top |\,\, \eta(\cdot):\mathbb{R} \rightarrow \mathbb{R} \right\},
    \end{equation}
    where, with a slight abuse of notation, $\eta(\tilde{\Lambda})$ denotes the pointwise application of $\eta$ to the entries of the diagonal matrix $\tilde{\Lambda}$.
   
\end{enumerate}

We now evaluate the denoising maps $f$ in the classes defined above under the loss in~\eqref{eq::problem}, replacing the full attention-based estimator $g$ with $f$. Propositions~\ref{prop: minimum lin loss} and~\ref{prop: minimum filter loss} characterize the minimum achievable loss for standard and spectral attention scores, respectively. Their proofs are given in Appendices~\ref{app: prop minimum lin loss proof} and~\ref{app: prop minimum filter loss proof}.

\begin{proposition}\label{prop: minimum lin loss}
    For some constant $c \geq 0$, the minimum loss achievable with standard attention scores,  $\ell^\star_{\text{A}} := \min_{f \in \mathcal{F}_{\text{A}}} \ell(f)$,  is given by
    \begin{equation}
    \ell^\star_{\text{A}} = \mathrm{tr}\left(\mathrm{Cov}(\mathrm{vec}(B))\right) + c,\,
    \end{equation}
    where $B := \tilde{U}^\top A \tilde{U}$. Furthermore, the minimizer is given by $f(\cdot\,; W^\star)$ where
    \begin{equation}
        W^\star = \mathbb{E}_{A,\,\mathcal{E}} [B].
    \end{equation}
\end{proposition}

\begin{proposition}\label{prop: minimum filter loss}
    For some constant $c \geq 0$, the minimum loss achievable with spectral attention scores, $\ell^\star_{\text{SA}} := \min_{f \in \mathcal{F}_{\text{SA}}} \ell(f)$, is given by
   \begin{equation}
    \ell^\star_{\text{SA}} = \mathrm{tr}\left(\mathrm{Cov}(\mathrm{vec}(B) \mid \tilde{\Lambda})\right) + c,
\end{equation}
    where $B := \tilde{U}^\top A \tilde{U}$. Furthermore, the minimizer is given by $f(\cdot\,; h^\star)$ where
\begin{equation}
    h^\star(\tilde{\Lambda}) = \mathbb{E}_{A,\,\mathcal{E}}\!\left[B \;\middle|\; \tilde{\Lambda}\right].
\end{equation}    
\end{proposition}

By the law of total variance, conditioning can only reduce variance, so $\ell^\star_{\text{SA}}$ is smaller or equal to $\ell^\star_{\text{A}}$. We formalize this in Theorem~\ref{thm: loss improvement}, proven in Appendix~\ref{app: proof loss improvement}.

\begin{theorem}\label{thm: loss improvement}
The improvement in the minimum achievable loss by the spectral attention scores compared to the standard attention scores is given by
\begin{equation}\label{eq: spectral diversity}
\ell^\star_{\text{A}} - \ell^\star_{\text{SA}} = \text{\normalfont{SD}}(\mathcal{D}_{\mathcal{G}}) := \mathbb{E}\left[\left\|\mathbb{E}[B \mid \tilde{\Lambda}] - \mathbb{E}[B]\right\|_F^2\right] = \mathrm{tr}(\mathrm{Cov}(\mathbb{E}[\mathrm{vec}(B) \mid \tilde{\Lambda}])),
\end{equation}
where we define $\text{\normalfont{SD}}(\mathcal{D}_{\mathcal{G}})$ as the "spectral diversity" which depends on the graph data distribution and on the noise distribution.
\end{theorem}

\begin{remark}
    The spectral diversity $\text{\normalfont{SD}}(\mathcal{D}_{\mathcal{G}})$ essentially quantifies how much information we gain about the clean graph by observing the noisy spectrum. For more intuition on this, if we assume the noise is sufficiently small so that the noisy eigenvectors are close to the clean ones, then from~\eqref{eq: spectral diversity} we can write $\mathrm{tr}(\mathrm{Cov}(\mathbb{E}[\mathrm{vec}(B) \mid \tilde{\Lambda}])) \approx \mathrm{tr}(\mathrm{Cov}(\mathbb{E}[\Lambda \mid \tilde{\Lambda}]))$ which is akin to the variance of the expected value of $\Lambda$ having observed $\tilde \Lambda$. At one extreme, if the graph distribution is degenerate, i.e., its support is only a single graph, then $\mathbb{E}[\Lambda \mid \tilde{\Lambda}] = \mathbb{E}[\Lambda]$, and its variance and therefore the spectral diversity are $0$. At the other extreme, the more $\Lambda$ varies across graphs, the more $\mathbb{E}[\Lambda \mid \tilde{\Lambda}]$ can vary, leading to larger variance and spectral diversity.
\end{remark}
    
\subsection{From Spectral Attention to Graph Convolutional Attention}

The spectral attention class is rather abstract as it includes any arbitrary function of the eigenvalues. GCA on the other hand, is practically realizable with graph convolutional filters. In this section we aim to bridge the gap between these two classes in certain settings. We begin by defining a graph convolutional filter bank acting on input graph features $X \in \mathbb{R}^{n \times d}$ as
\begin{equation}\label{eq: graph filter bank}
    h_{\text{filt}}(X) = \sum_{l = 0}^{L-1} S^l X H^{(l)},
\end{equation}
where $S$ is the graph-shift operator (e.g., adjacency or Laplacian of the graph) and $H^{(l)} \in \mathbb{R}^{d \times d}$ is a set of learnable weights. In our case the input features are the eigenvectors of the graph, i.e., $X = U$, thus when $S=A$, Eq.~\eqref{eq: graph filter bank} simplifies to
\begin{equation}\label{eq: graph filter bank input eigenvec}
    h_{\text{filt}}(U) = \sum_{l = 0}^{L - 1} U \Lambda^l U^\top U H^{(l)} = U \left( \sum_{l = 0}^{L - 1} \Lambda^l H^{(l)} \right).
\end{equation}
By replacing the query and key projections with graph convolutional filter banks as in~\eqref{eq: graph filter bank input eigenvec} we get the following function of the adjacency

\begin{equation}\label{eq:productGF}
    f(A; H) = U \left( \sum_{l = 0}^{L - 1} \Lambda^l H_Q^{(l)} \right) \left( \sum_{l = 0}^{L - 1} \Lambda^l H_K^{(l)} \right)^\top U^\top,
\end{equation}

where $H_K^{(l)}$ and $H_Q^{(l)}$ are diagonal matrices. Eq.~\eqref{eq:productGF} reduces to a product of graph filters.
The connection between graph convolutional filters and the GCA class defined in~\eqref{eq: GCA definition} can be made formal. We formalize this connection in Lemma~\ref{lem:1} (proof in Appendix \ref{app:lem1}).

\begin{lemma}\label{lem:1}
    A graph convolutional filter bank, i.e., $h_{\text{filt}}(\Lambda)=\left(\sum_{l=0}^{L-1}\Lambda^l H_Q^{(l)}\right)\left(\sum_{l=0}^{L-1}\Lambda^l H_K^{(l)}\right)^\top$, for sufficiently large $L$,  can approximate any function acting pointwise on the eigenvalues $\eta(\lambda): \mathbb{R} \rightarrow \mathbb{R}$,  if $\eta$ is continuous on any compact interval $[a,b] \subset \mathbb{R}$.
\end{lemma}

With the preceding discussion and Lemma~\ref{lem:1} we have established a practical way to realize the functions in $\mathcal{F}_{\text{GCA}}$. Still in general there could be a gap between the minimum achievable loss of the graph convolutional attention model and the spectral attention model because the minimizer of the loss in $\mathcal{F}_{\text{SA}}$ is not necessarily a pointwise function of the eigenvalues. In other words since $\mathcal{F}_{\text{GCA}} \subset \mathcal{F}_{\text{SA}}$, it follows that $\ell^\star_{\text{SA}} \leq \ell^\star_{\text{GCA}}$. Thus, it is not obvious whether GCA also improves upon standard attention scores. In order to bridge this gap we show in this section that under certain conditions we have $\ell^\star_{\text{GCA}} = \ell^\star_{\text{SA}}$.

\paragraph{Stochastic Block Model (SBM).} Consider the general SBM of size $n$ with $k$ communities. A given graph sampled
from this model can be written as
\begin{equation}
    A = \mathbb{E} \left[ A \right] + \underbrace{\left( A - \mathbb{E} \left[ A \right] \right)}_{\Delta},
\end{equation}
where $\mathbb{E} \left[ A \right]$ is a deterministic rank $k$ matrix with block structure and
the matrix $\Delta := A - \mathbb{E} \left[ A \right]$ encodes the randomness. 
Now consider our noisy observation $\tilde{A} = A + \mathcal{E}$ which we can write as
\begin{equation}
    \tilde{A} = \mathbb{E} \left[ A \right] + \Delta + \mathcal{E}.
\end{equation}

\begin{assumption}\label{ass: low variance sbm}
We assume the variance of $\Delta$ is small compared to the eigenvalues of $\mathbb{E} \left[ A \right]$. In this setting we can ignore $\Delta$ and focus on recovering
the low rank signal $A \approx \mathbb{E} \left[ A \right]$ from the noisy observation $\mathbb{E}[A] + \mathcal{E}$. We further consider $\mathcal{E}$ to be a standard Wigner matrix scaled by $v$.
\end{assumption}

A Wigner matrix is a symmetric random matrix with independent zero-mean entries on and above the diagonal. This assumption is not overly restrictive in our setting, since it captures independent entrywise perturbations of a graph. Gaussian noise with independent entries is a Wigner perturbation. Additionally, edge-flipping noise in the SBM setting, can be decomposed into a zero-mean Wigner-type component plus a deterministic bias term (See~\ref{app:edge_flipping} for proof). Hence the edge-flip setting also reduces to a low-rank deterministic matrix plus a Wigner-noise matrix. See Appendix~\ref{app: low variance sbm shrinkage proof} for proof.

\begin{theorem}[Outlier Shrinkage]\label{thm: low variance sbm shrinkage}
    Under the setting of Assumption~\ref{ass: low variance sbm}, we have
    \begin{equation}
        \lim_{n \to \infty} \mathbb{E} \left[  \tilde{U}^\top A \tilde{U}  \middle| \tilde{\Lambda} \right] = \eta_{out}(\tilde{\Lambda}),
    \end{equation}
    where $\eta_{out}(\cdot):\mathbb{R} \rightarrow \mathbb{R}$ is a function acting pointwise on the noisy eigenvalues given by
        \begin{equation}
    \eta_{out}(\tilde{\lambda}) = \begin{cases}
        \sqrt{\tilde{\lambda}^2 - 4 v^2} & \text{if } |\tilde{\lambda}| > 2v \\
        0 & \text{if } |\tilde{\lambda}| \leq 2v.
    \end{cases}
\end{equation}
As a result, $\eta_{out}(\cdot) \in \mathcal{F}_{\text{GCA}}$.
\end{theorem}

\begin{remark}
It follows immediately from Theorem~\ref{thm: low variance sbm shrinkage} and Lemma~\ref{lem:1} that in the asymptotic regime considered above, $\ell^\star_{\text{GCA}} = \ell^\star_{\text{SA}}$.
Consequently, in this regime GCA attains the optimal spectral-attention loss and therefore improves on standard attention scores by $\text{SD}(\mathcal{D}_\mathcal{G})$.
\end{remark}

\section{How Softmax Can Help Denoise Further}\label{sec: nonlinear theory}

In Section~\ref{sec: linear theory}, we showed how a model linear in the input tokens $\tilde U$ denoises the eigenvalues using graph-convolutional filters. We now turn from the attention scores $f$ to the full denoising map $g$, and examine how the softmax operation and the outer product of the resulting representations can further improve denoising.
We will see that, from a spectral perspective, these operations reduce the error associated with the noisy \emph{eigenvectors}, complementing the eigenvalue denoising studied in Section~\ref{sec: linear theory}.

\paragraph{Problem Setup.} Let $A \in \mathbb{R}^{n \times n}$ be block diagonal with $k$ blocks of sizes $m_1, \ldots, m_k$ respectively. All entries within the blocks equal one and all others equal zero. Hence, $A$ has rank $k$; we denote by $U_k$ the top $k$ eigenvectors. We denote by $\hat{A} = \tilde{U} \Lambda \tilde{U}^\top$ the matrix with clean eigenvalues and noisy eigenvectors.
Given the attention-based estimator introduced in \eqref{eq::graph_based_att}, we consider the ideal case where the spectral denoising step studied in Section~\ref{sec: linear theory} is exact,  and set $W_V W_V^\top = \frac{1}{\alpha} W_Q W_K^\top = \Lambda$.
With this choice of $W_Q, W_K, W_V$, the attention-based denoising map from~\eqref{eq::graph_based_att}, simplifies into an analytically tractable surrogate, which we use to study denoising under the assumption that the clean eigenvalues (and hence $\hat{A}$) are known
\begin{equation}\label{eq: surrogate attention model}
    g_{\text{sm}}(\hat{A}; \alpha) = \sigma(\alpha\hat{A}) \cdot \hat{A} \cdot (\sigma(\alpha\hat{A}))^\top,
\end{equation}
where $\sigma$ applies a row-wise softmax operation. The model in~\eqref{eq: surrogate attention model} isolates the role of the softmax and the outer product, with the parameter $\alpha$ controlling how sharp the softmax is.
We define the $\alpha$-dependent denoising loss as
\begin{equation}
    \ell_{\text{\normalfont{soft}}}(\alpha) = \,\,\mathbb{E}_{\mathcal{E}} \left\| g_{\text{sm}}(\hat{A}; \alpha) - A \right\|_F^2,
\end{equation}
where $\mathcal{E}$ is the noise matrix. Our goal is to show that the minimum of this loss is smaller than the loss when we have the clean eigenvalues with the noisy eigenvectors. We denote the latter as $\ell_{\text{\normalfont{basis}}}$, since it is essentially the loss caused by the mismatch between the noisy eigenbasis $\tilde U$ and the original $U$:
\begin{equation}
    \ell_{\text{\normalfont{basis}}} = \,\,\mathbb{E}_{\mathcal{E}} \left\| \hat{A} - A\right\|_F^2.
\end{equation}
Next, we show that the softmax denoiser approximates projection onto the clean eigenbasis, which itself improves over $\ell_{\text{\normalfont{basis}}}$. Concretely, we proceed in two steps. First, in Lemma~\ref{lem: projection helps} we show that projection onto the clean eigenbasis strictly improves over $\ell_{\text{\normalfont{basis}}}$. Then, in Theorem~\ref{thm: softmax reduce error} we show that for a suitable choice of $\alpha$, the softmax denoiser approximates this projection well enough to inherit the improvement.
The connection between softmax and the aforementioned projection becomes clearer by noting that by replacing $\hat{A}$ in the softmax denoiser with the clean adjacency matrix $A$, we get
\begin{equation}
    \lim_{\alpha \to \infty} \sigma(\alpha A) = P_U,
\end{equation}
where $P_U = U_k U_k^\top$ is the projection matrix onto the span of the top-$k$ eigenvectors of $A$. This motivates studying the loss attained by projecting onto the clean eigenbasis
\begin{equation}
    \ell_{\text{\normalfont{proj}}} = \,\,\mathbb{E}_{\mathcal{E}} \left\| P_U \hat{A} P_U^\top - A \right\|_F^2.
\end{equation}
The following lemma shows that this projection strictly reduces the loss, since it removes the components of the noise caused by the eigenbasis rotating outside of the clean subspace. The proof can be found in Appendix~\ref{app: proof of projection helps}

\begin{lemma}[Projection improves denoising]\label{lem: projection helps} Projecting onto the principal eigenbasis of the original graph strictly reduces the loss, i.e.,
    \begin{equation}
       \ell_{\text{\normalfont{proj}}} <  \ell_{\text{\normalfont{basis}}}.
    \end{equation}
\end{lemma}

When denoising we do not have access to the clean adjacency matrix $A$, only to $\hat{A}$. The softmax denoiser $\sigma(\alpha \hat{A})$ can therefore be viewed as an \emph{approximate} projection: when $\hat{A}$ is close to $A$, applying softmax with sufficiently large $\alpha$ yields a matrix close to $P_U$. The gap between $\ell_{\text{\normalfont{soft}}}(\alpha)$ and $\ell_{\text{\normalfont{proj}}}$ depends on the choice of $\alpha$: for very small or very large $\alpha$ the gap is large, but the following theorem shows that there exist intermediate choices of $\alpha$ for which the gap is smaller than the improvement gained by projection i.e., $\ell_{\text{\normalfont{basis}}} -\ell_{\text{\normalfont{proj}}}$, which in turn means that the softmax denoiser strictly improves upon $\ell_{\text{\normalfont{basis}}}$. The proof can be found in Appendix~\ref{app: proof of softmax reduce error}.

\begin{theorem}\label{thm: softmax reduce error}
     Assume that the noise is sufficiently small, i.e., $ \;\mathbb{E}\left[\|\mathcal{E}\|_F^2\right] \;\ll\; \min\!\left(1,\; \frac{m^2}{\log^2 k}\right)\;$ and that $m_i=m$ for $i=0,...,k$. Then, there exists $\alpha^\star > 0 $ such that
    \begin{equation}
        \ell_{\text{\normalfont{soft}}}(\alpha^\star) - \ell_{\text{\normalfont{proj}}} < \ell_{\text{\normalfont{basis}}} - \ell_{\text{\normalfont{proj}}}.
    \end{equation}
    
\end{theorem}

\section{Numerical Experiments}
\label{sec:exp_new}

\begin{figure}[t]
    \centering
    \begin{subfigure}{0.49\textwidth}
        \centering
        \includegraphics[width=\linewidth]{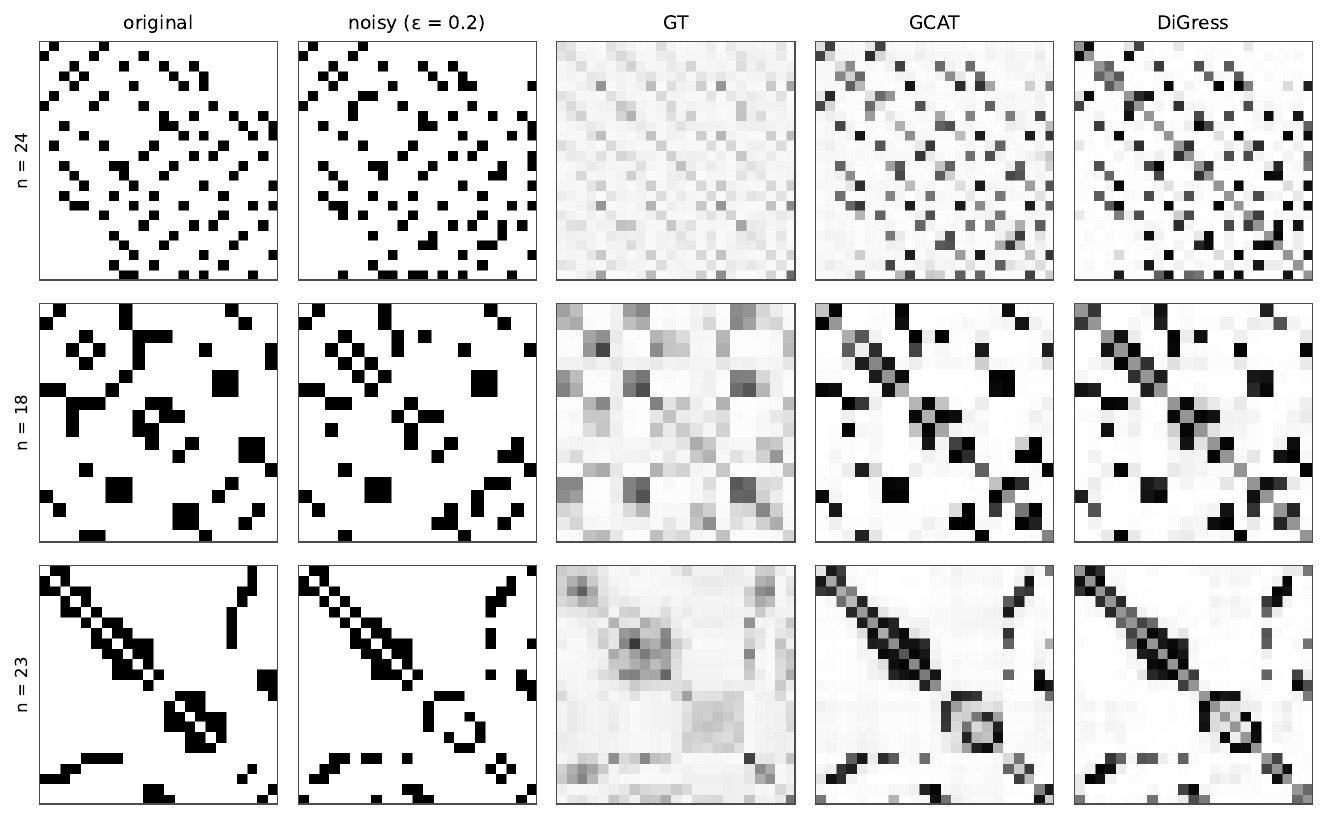}
        \caption{ENZYMES}
        \label{fig:enzymes denoise panels}
    \end{subfigure}
    \hfill
    \begin{subfigure}{0.49\textwidth}
        \centering
        \includegraphics[width=\linewidth]{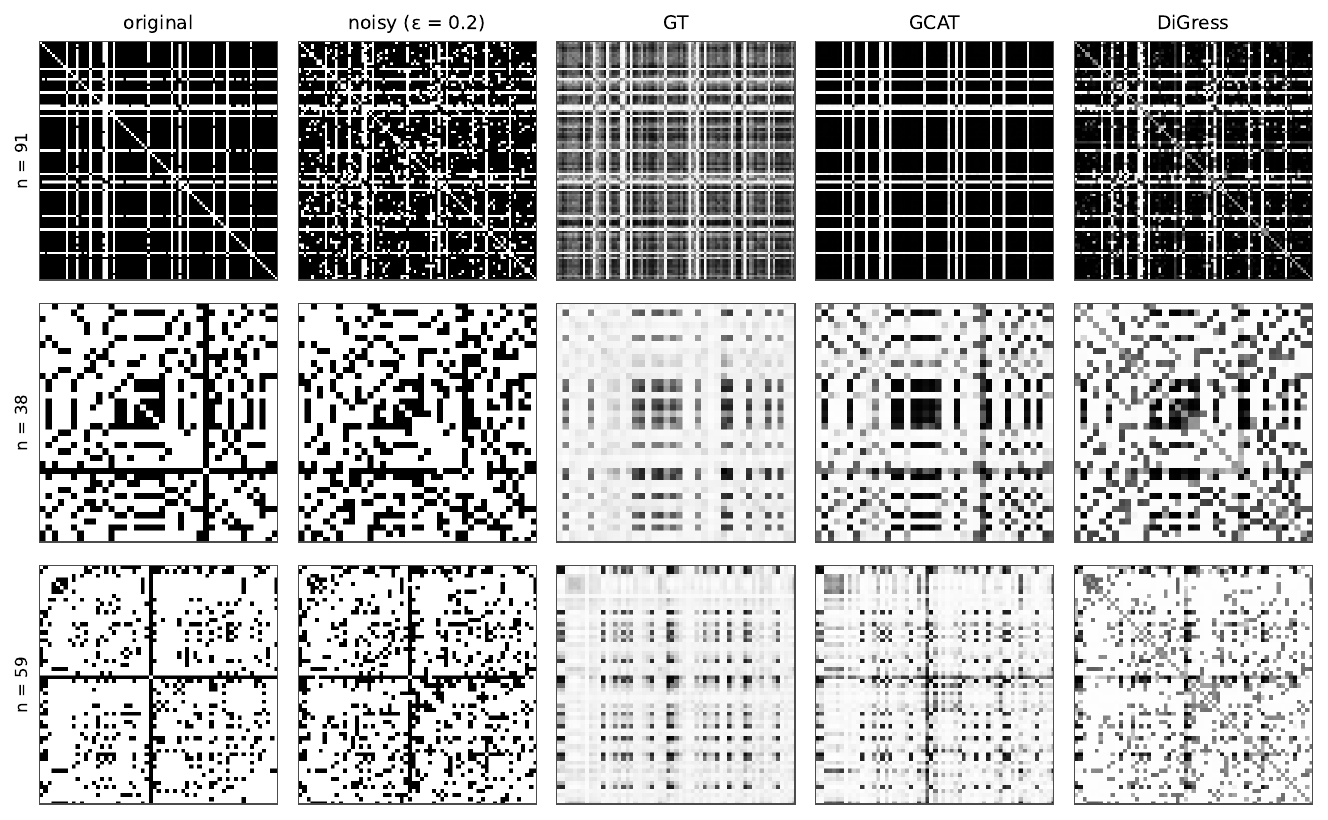}
        \caption{COLLAB}
        \label{fig:collab denoise panels}
    \end{subfigure}
    \caption{Denoising at $\varepsilon = 0.2$. Each row is a validation graph: its clean adjacency,
    the corrupted input shared by all three models, and the edge probabilities each model predicts
    (darker is higher).}
    \label{fig:qualitative denoising}
\end{figure}

Our experiments aim to assess whether the theoretical advantages of GCA translate into improved graph denoising and generation. We first compare a standard graph transformer (GT) with a graph convolutional attention transformer (GCAT), that is, a transformer that uses GCA in its attention layers. We examine whether the denoising gains follow the spectral diversity predicted by Theorem~\ref{thm: loss improvement} and whether they extend to graph generation. Finally, we compare GCAT with the DiGress architecture in terms of generation quality and inference time, showing that GCAT achieves comparable quality with faster inference.

\textbf{Models.} Our baseline, GT, is a graph transformer with $L$ multi-head attention blocks. Its input tokens are the top $k$ eigenvectors of the noisy adjacency matrix $\tilde A$, without additional node or edge features. The adjacency logits are obtained from a learned linear combination of the per-head attention scores in the final block. GCAT uses the same architecture and differs only in the query and key projections. At layer $\ell$, GT uses
\begin{equation}
Q^{(\ell)}=X^{(\ell)}W_Q^{(\ell)},
\qquad
K^{(\ell)}=X^{(\ell)}W_K^{(\ell)},
\end{equation}
whereas GCAT uses
\begin{equation}
Q^{(\ell)}
=\sum_{p=0}^{P-1}\tilde A^{\,p}X^{(\ell)}H_{Q,p}^{(\ell)},
\qquad
K^{(\ell)}
=\sum_{p=0}^{P-1}\tilde A^{\,p}X^{(\ell)}H_{K,p}^{(\ell)}.
\end{equation}
Thus, GCAT forms queries and keys using learnable graph filters on the noisy input. Architectural details are given in Appendix~\ref{app: arch details GT GCAT}.
For the generative experiments, we use the diffusion framework of DiGress~\citep{vignac2022digress} with three denoising architectures: GT, GCAT, and the original DiGress architecture. We match their parameter counts and provide no additional structural features to any of the models, allowing us to compare the effect of the denoising architecture within the same generative framework.

\textbf{Datasets.} We use real benchmark  datasets from
PyTorch Geometric's \texttt{TUDataset} collection except one. All graphs are treated as undirected and unattributed: we
keep only the adjacency structure and discard any node or edge labels
supplied by the source dataset. We use five collections of variable-size unattributed graphs, all from \texttt{TUDataset} except DEEZER-EGO-NETS. PROTEINS~\citep{borgwardt2005proteins} has $1{,}113$ protein-structure graphs (up to $620$ nodes) with secondary-structure elements as nodes. ENZYMES~\citep{borgwardt2005proteins} has $600$ enzyme tertiary-structure graphs (up to $126$ nodes) across six EC classes. IMDB-BINARY~\citep{yanardag2015deep} provides $1{,}000$ movie co-appearance ego networks (up to $136$ nodes, small and dense), and COLLAB~\citep{yanardag2015deep} provides $5{,}000$ scientific-collaboration ego networks from three physics communities (up to $492$ nodes, heavy-tailed). DEEZER-EGO-NETS~\citep{rozemberczki2020characteristic} contains $9{,}629$ Deezer user ego networks with mutual-friendship edges (up to $363$ nodes).

\textbf{Experiments and metrics.}
We assess denoising through validation loss and compare the predicted edge probabilities with the clean adjacency matrices. To relate the denoising results to Theorem~\ref{thm: loss improvement}, we estimate spectral diversity using the KNN procedure described in Appendix~\ref{app:knn}.
For generation, following~\citep{you2018graphrnn}, we compare the distributions of node degrees, clustering coefficients, four-node orbits, and normalized Laplacian eigenvalues using MMD$^2$. For each feature $f$, let $m^{(f)}(M)$ denote the median MMD$^2$ between graphs generated by model $M$ and real validation graphs over 25 scoring repeats. We compute a real-real reference $\mathrm{rr}^{(f)}$ from two disjoint samples of validation graphs. We also compute an Erdős-Rényi reference $\mathrm{er}^{(f)}$ using graphs drawn from $G(n,p)$ with the node counts and edge densities of the second validation sample. We normalize each score as
\[
s_{\mathrm{norm}}^{(f)}(M)
=\frac{m^{(f)}(M)-\mathrm{rr}^{(f)}}
       {\mathrm{er}^{(f)}-\mathrm{rr}^{(f)}}.
\]
We then report the difference $\Delta_{\mathrm{norm}}^{(f)}
=s_{\mathrm{norm}}^{(f)}(\mathrm{GT})
-s_{\mathrm{norm}}^{(f)}(\mathrm{GCAT})$, so positive values favor GCAT. We mark a comparison as N/A when the Erdős-Rényi median does not exceed the 90th percentile of the real-real scores, since the references are then insufficiently separated. For the three-seed GT-GCAT comparison, we average model scores across seeds before taking their difference. Sampling and kernel details are given in Appendix~\ref{app:eval_protocol}.

\begin{figure}[H]
    \centering
    \begin{subfigure}{0.49\textwidth}
        \centering
        \includegraphics[width=\linewidth]{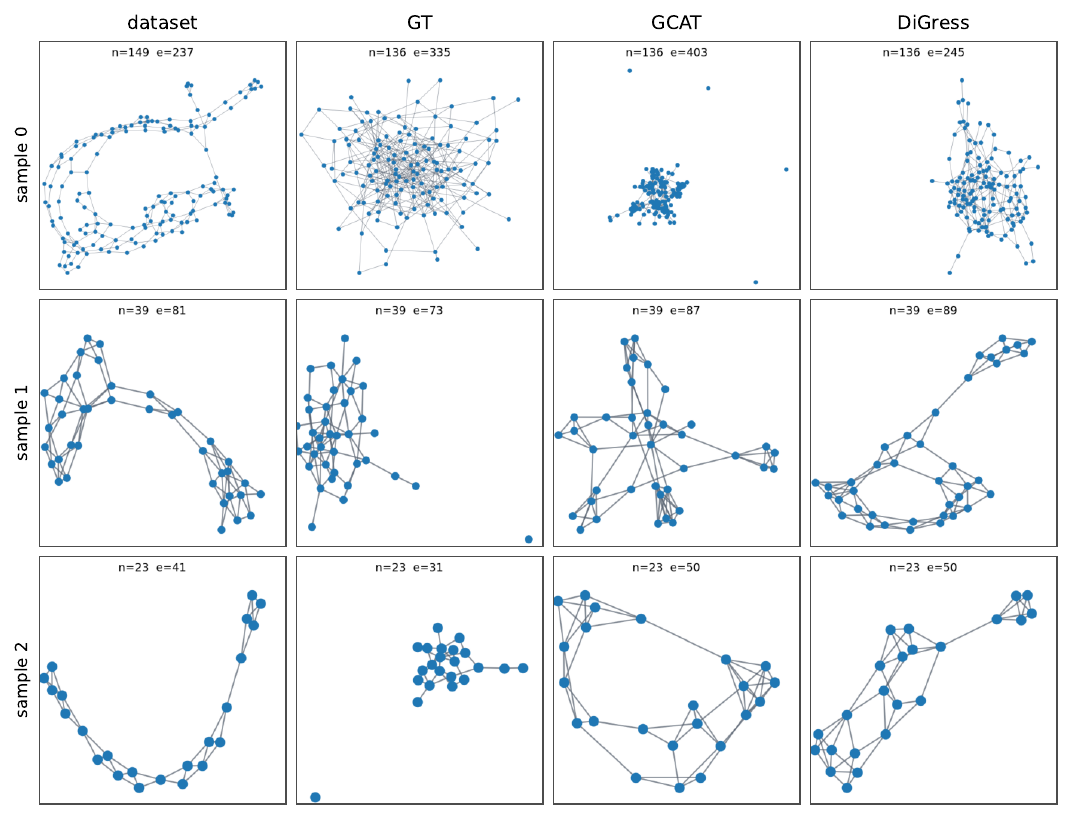}
        \caption{PROTEINS}
        \label{fig:proteins samples}
    \end{subfigure}
    \hfill
    \begin{subfigure}{0.49\textwidth}
        \centering
        \includegraphics[width=\linewidth]{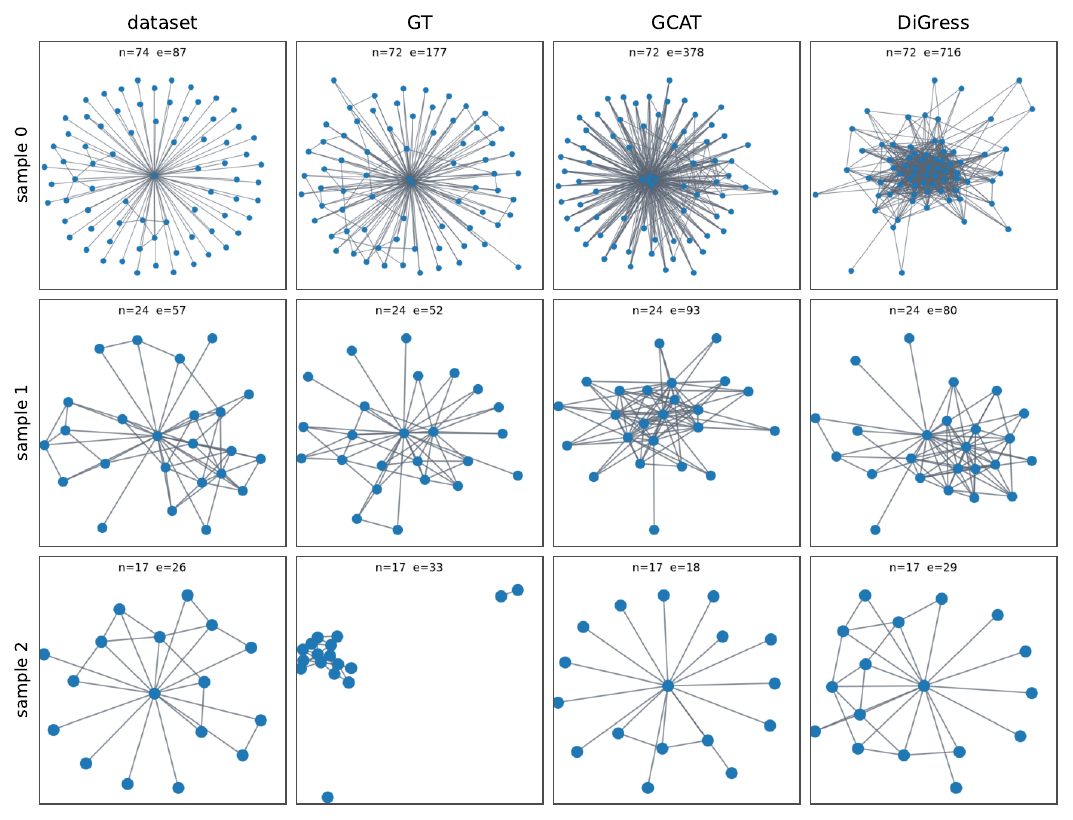}
        \caption{DEEZER-EGO-NETS}
        \label{fig:deezer samples}
    \end{subfigure}
    \caption{Samples from each model beside validation graphs of matching size; rows are
    independent draws.}
    \label{fig:qualitative generation}
\end{figure}

\textbf{Results.}
GCAT achieves lower denoising validation loss than GT across the five datasets. The examples in Figure~\ref{fig:qualitative denoising} illustrate this difference on ENZYMES and COLLAB. Given the same corrupted graph, GCAT concentrates its predicted edge probabilities more closely on the structure of the clean graph. Moreover, Figure~\ref{fig:gain vs ghat five} shows that datasets with higher estimated spectral diversity tend to exhibit larger denoising gains. The correlation is consistent with Theorem~\ref{thm: loss improvement}, although PROTEINS falls below the overall trend.
The denoising gains also carry over to generation. Figure~\ref{fig: GT vs GCAT} shows that GCAT improves on GT across most combinations of dataset and graph statistic. 

\begin{figure}[h]
    \centering
    \begin{subfigure}{0.49\textwidth}
        \centering
        \includegraphics[width=\linewidth]{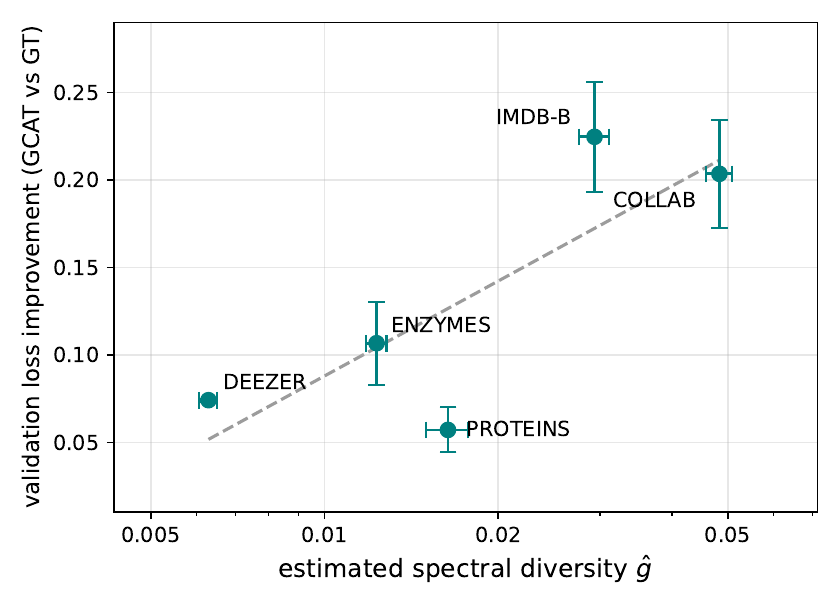}
        \caption{Improvement against spectral diversity}
        \label{fig:gain vs ghat five}
    \end{subfigure}
    \hfill
    \begin{subfigure}{0.49\textwidth}
        \centering
        \includegraphics[width=\linewidth]{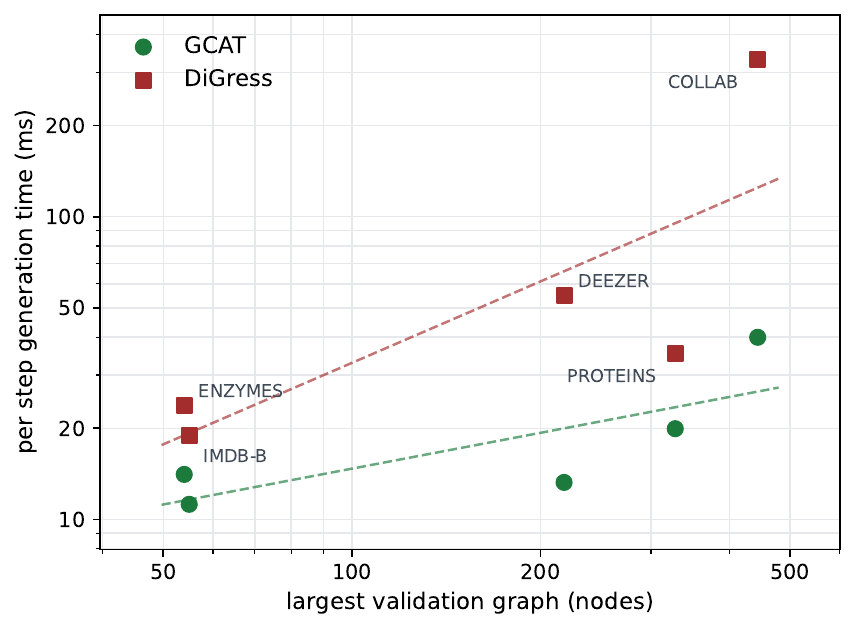}
        \caption{Generation cost against graph size}
        \label{fig: per reverse step wall time}
    \end{subfigure}
    \caption{\textbf{(a)} Validation-loss improvement of GCAT over GT against the estimated spectral
    diversity $\widehat{\text{SD}}({\mathcal{D}_\mathcal{G}})$ (Appendix~\ref{app:knn}); error bars are standard deviations over seeds and the
    dashed line a least-squares fit in $\log \widehat{\text{SD}}({\mathcal{D}_\mathcal{G}})$ ($r = 0.81$). \textbf{(b)} Time per
    reverse-diffusion step against the largest validation graph, for capacity-matched models
    ($\approx$670k parameters); dashed lines are log--log fits, $n^{0.39}$ for GCAT and $n^{0.89}$
    for DiGress.}
    \label{fig:diversity and cost}
\end{figure}

\begin{figure}[h]
    \centering
    \begin{subfigure}{0.49\textwidth}
        \centering
        \includegraphics[width=\linewidth]{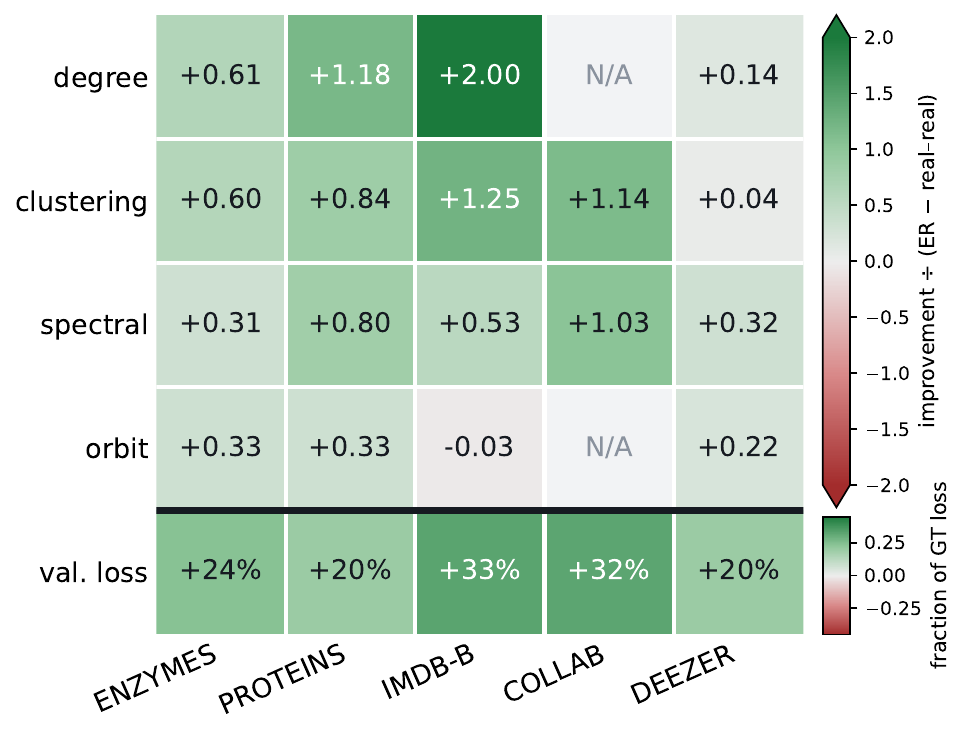}
        \caption{GCAT against GT}
        \label{fig: GT vs GCAT}
    \end{subfigure}
    \hfill
    \begin{subfigure}{0.49\textwidth}
        \centering
        \includegraphics[width=\linewidth]{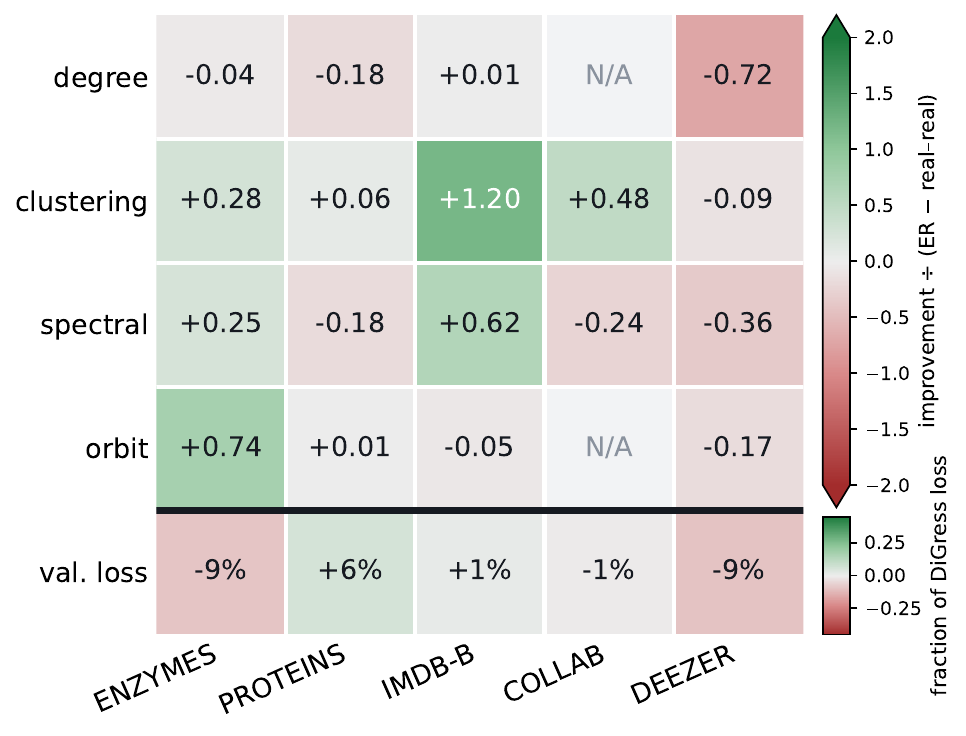}
        \caption{GCAT against DiGress}
        \label{fig: Digress vs GCAT}
    \end{subfigure}
    \caption{Generation quality; green means the first-named model is closer to the data. The MMD
    rows show the difference in reference-normalized MMD$^2$ (Appendix~\ref{app:eval_protocol}),
    N/A where the two references do not separate; below the rule, on its own scale, is the
    validation-loss improvement as a fraction of the baseline's loss. \textbf{(a)} GCAT against
    GT, mean of three seeds. \textbf{(b)} GCAT against DiGress at matched capacity.}
    \label{fig:generation heatmaps}
\end{figure}

When compared with the DiGress architecture at matched capacity, the results are mixed rather than uniformly favoring either architecture. As shown in Figure~\ref{fig: Digress vs GCAT}, GCAT performs better on ENZYMES and IMDB-BINARY, DiGress performs better on DEEZER-EGO-NETS, and the two are broadly comparable on COLLAB and PROTEINS. Figure~\ref{fig:qualitative generation} provides examples of generated graphs alongside validation graphs.
However, Figure~\ref{fig: per reverse step wall time} reveals an important difference in inference cost. While GCAT and the DiGress architecture achieve comparable generation quality, GCAT requires less time per reverse-diffusion step across the evaluated graph sizes, with the gap increasing for larger graphs. Reporting time per step isolates the cost of evaluating the denoiser from the number of diffusion steps. Both architectures use attention over node pairs, but the DiGress architecture also updates edge representations throughout the network, increasing the cost of each evaluation. Thus, GCAT is more scalable, i.e., it achieves comparable generation quality with faster inference in the evaluated setting.

\section{Conclusions}

We studied graph denoising from a spectral perspective to understand how attention-based architectures operate in graph diffusion models. Under a denoising objective, we showed that standard attention scores learn an average spectral rule across the training distribution and become suboptimal when graph spectra vary. We introduced spectral attention scores, characterized their advantage in terms of spectral diversity, and proposed graph convolutional attention (GCA) as a practical realization based on graph-filtered queries and keys. For asymptotically large SBMs, GCA matches the optimal loss of spectral attention under suitable conditions. We further showed that softmax provides additional denoising (for a suitable $\alpha$) by approximately projecting noisy eigenvectors onto the clean eigenspace.
Across synthetic and real datasets, replacing standard query--key projections with GCA improved graph denoising, with gains correlated with spectral diversity. These gains also extended to graph diffusion. Furthermore, replacing the DiGress architecture with a simpler GCA-based architecture within the same generative framework achieved comparable graph generation quality with faster inference.
Overall, our results suggest that spectral adaptation is a key ingredient for effective graph denoising and generative diffusion, and provide both a theoretical explanation and a practical architectural principle for designing stronger graph diffusion models.



\clearpage
\bibliography{dd-bib}
\bibliographystyle{iclr2027_conference}

\clearpage
\appendix

\section{Limitations and Future Directions}
\label{app:limitations}

The limitations of this work concern the scope of the experiments and the assumptions of the theoretical analysis. Our experiments cover a limited set of graph datasets. Testing GCA on more real and synthetic graphs would help identify settings in which it is effective and cases in which it fails. We also use eigenvectors as input features in most of our analysis. Studying how GCA interacts with other positional encodings, such as R-PEARL~\citep{kanatsoulis2025learning}, is a natural next step.
Our theoretical results for GCA rely in part on stochastic block models. Extending the analysis to other graph families could clarify whether the same spectral mechanism explains denoising gains more generally. It would also be useful to understand theoretically why the DiGress architecture achieves generation quality comparable to GCAT despite their different designs.
Finally, the generative comparison is limited in its number of training seeds, datasets, and hyperparameter settings. Its results should therefore be viewed as preliminary. A broader evaluation is needed to determine how consistently GCAT matches the generation quality of DiGress while reducing inference time.

\section{Experimental Details and Additional Results}\label{app: experiments}

\subsection{Datasets}
\label{sec:datasets}

We evaluate on a mix of \textbf{real benchmark graphs} drawn from
PyTorch Geometric's \texttt{TUDataset} collection
and \textbf{synthetic SBMs} sampled with controllable block-size
heterogeneity. All graphs are treated as undirected and unattributed: we
keep only the adjacency structure and discard any node or edge labels
supplied by the source dataset. Unless stated otherwise, we use a
$70/10/20$ train/validation/test split with a fixed seed.

\subsubsection{Real datasets}
\label{sec:datasets:real}

We use five collections of variable-size unattributed graphs.

\begin{itemize}
  \item \textbf{PROTEINS}~\citep{borgwardt2005proteins}, from
    TUDataset: $1\,113$ protein-structure graphs in which nodes are
    secondary-structure elements connected when they are neighbours in
    the amino-acid sequence or in 3D space. Graph size ranges up to
    $620$ nodes, with most graphs in the few-tens to low-hundreds
    range.
  \item \textbf{ENZYMES}~\citep{borgwardt2005proteins}, from TUDataset:
    $600$ enzyme tertiary-structure graphs spanning the six EC
    top-level classes. Graph size ranges up to $126$ nodes, typically
    $\approx 30$--$60$ nodes per graph.
  \item \textbf{IMDB-BINARY}~\citep{yanardag2015deep}, from TUDataset:
    $1\,000$ movie-collaboration ego networks in which nodes are
    actors/actresses and edges connect those that appeared together in
    a film. Graph size ranges up to $136$ nodes with a strongly skewed
    distribution toward smaller, dense ego networks.
  \item \textbf{COLLAB}~\citep{yanardag2015deep}, from TUDataset:
    $5\,000$ scientific-collaboration ego networks built from three
    physics communities. Graph size ranges up to $492$ nodes, with
    heavy-tailed size and degree distributions.
  \item \textbf{DEEZER-EGO-NETS}~\citep{rozemberczki2020characteristic}:
    $9\,629$ ego networks of Deezer users, with edges representing
    mutual-friend relationships in the ego's neighbourhood. Graph size
    ranges up to $363$ nodes; the distribution is dominated by
    ego networks of a few dozen nodes.
\end{itemize}

In all five cases we drop the original node and edge attributes and
predict the binary adjacency only. This matches the unconditional
graph-generation setting used in prior work~\citep{
you2018graphrnn,liao2019efficient,vignac2022digress}. Splits use the
PyG-default ordering with a fixed seed of $42$ and a $70/10/20$
train/val/test ratio. Every batch is padded to the dataset's published
per-graph maximum ($n_{\max}$ in Table~\ref{tab:dataset_stats}), which
also bounds the eigendecomposition. Table~\ref{tab:dataset_stats}
summarises the validation split that every generative evaluation in
this paper is scored against.

\begin{table}[h]
  \centering
  \small
  \caption{The five datasets of the generative comparison. ``$n_{\max}$''
    is the published per-graph maximum used for padding; the last three
    columns describe the validation split, which supplies the reference
    graphs for every MMD$^2$ and the graphs on which validation losses are
    measured. Because a sampling batch is padded to its largest member,
    the maximum validation size, not the mean, sets the per-step memory
    and time of DiGress's $n \times n$ edge tensors.}
  \label{tab:dataset_stats}
  \begin{tabular}{@{}lrrrrr@{}}
    \toprule
    Dataset & \# graphs & $n_{\max}$ & val.\ graphs & mean $n$ (val.) & max $n$ (val.) \\
    \midrule
    ENZYMES         &  $600$  & $126$ &  $60$ &  $27.4$ &  $54$ \\
    PROTEINS        & $1\,113$ & $620$ & $111$ &  $38.5$ & $328$ \\
    IMDB-BINARY     & $1\,000$ & $136$ & $100$ &  $19.0$ &  $55$ \\
    COLLAB          & $5\,000$ & $492$ & $500$ &  $74.2$ & $444$ \\
    DEEZER-EGO-NETS & $9\,629$ & $363$ & $962$ &  $23.2$ & $218$ \\
    \bottomrule
  \end{tabular}
\end{table}

\subsection{Architectural details}
\label{app:arch_details}

\subsubsection{Graph Transformer and Graph Convolutional Attention Transformer}\label{app: arch details GT GCAT}
\paragraph{Shared pipeline.}
Both models read the noisy adjacency $A \in \{0,1\}^{n \times n}$ and
nothing else: no node attributes, no hand-crafted structural features,
and no timestep. We extract the $k$ eigenvectors of $A$ with the largest
absolute eigenvalues, $V \in \mathbb{R}^{n \times k}$ (sign-fixed so
that the first non-zero entry of each column is positive), and lift
them to the model dimension $d$:
\begin{equation}
X^{(0)} = V\, W_{\mathrm{in}} + \mathbf{1}\, b_{\mathrm{in}}^{\top},
\qquad W_{\mathrm{in}} \in \mathbb{R}^{k \times d}.
\end{equation}
Each of the first $L-1$ blocks applies a residual multi-head attention
update with $\text{\normalfont{SD}}(\mathcal{D}_{\mathcal{G}})$ heads of dimension $d_k = d_v = d/h$:
\begin{equation}
X^{(\ell+1)} = \mathrm{LN}\!\left(
X^{(\ell)} + \mathrm{MHA}^{(\ell)}\!\big(A,\,X^{(\ell)}\big)
\right),
\qquad \ell = 0,\dots,L-2,
\end{equation}
where, given queries, keys, and values $Q^{(\ell)}, K^{(\ell)}, V^{(\ell)}$, each head computes
\begin{equation}
S_i^{(\ell)} = \frac{Q_i^{(\ell)} {K_i^{(\ell)}}^{\!\top}}{\sqrt{d_k}},
\qquad
\mathrm{Attn}_i^{(\ell)}
= \mathrm{softmax}\!\big(S_i^{(\ell)}\big)\,V_i^{(\ell)},
\qquad i = 1,\dots,h,
\end{equation}
and the heads are concatenated and mixed by an output projection
$W_O \in \mathbb{R}^{d \times d}$. The final block produces no node
features; its per-head scores $S^{(L-1)} \in \mathbb{R}^{h \times n \times n}$
are mixed into the edge logits by a learned linear combination,
\begin{equation}
\hat{A} = \sum_{i=1}^{h} w_i\, S_i^{(L-1)},
\qquad w \in \mathbb{R}^{h},
\end{equation}
and $\hat A_{uv}$ is read as the log-odds of an edge between $u$ and
$v$. The value projection is a plain linear map in both models; the two
models differ only in the $Q$/$K$ projections inside each block.

\textbf{Graph Transformer (GT)}
\begin{equation}
Q^{(\ell)} = X^{(\ell)} W_Q^{(\ell)},
\qquad
K^{(\ell)} = X^{(\ell)} W_K^{(\ell)}.
\end{equation}

\textbf{Graph Convolutional Attention Transformer (GCAT)}
\begin{equation}
Q^{(\ell)} = \mathrm{LN}\!\left(
\sum_{p=0}^{P-1} \bar A^{\,p}\, X^{(\ell)}\, W_{Q,p}^{(\ell)}
\right),
\qquad
K^{(\ell)} = \mathrm{LN}\!\left(
\sum_{p=0}^{P-1} \bar A^{\,p}\, X^{(\ell)}\, W_{K,p}^{(\ell)}
\right).
\end{equation}
The queries and keys become learnable polynomials of the noisy
adjacency, allowing two nodes to attend on the basis of their $P$-hop
neighbourhoods. With $P = 2$ taps, used throughout, each projection is
$X W_0 + \bar A X W_1$. The filter operator $\bar A$ is the raw
adjacency $A$ in the default GCAT; the capacity-matched wide variant
(Appendix~\ref{app:capacity}) uses the symmetrically normalised
$\bar A = D^{-1/2} A D^{-1/2}$ (degrees clamped at one), which keeps the
neighbourhood term degree-invariant and is what allows the wider model
to train stably on the dense datasets.

\paragraph{Sizes.} The default configuration has $k = 16$, $d = 64$,
$\text{\normalfont{SD}}(\mathcal{D}_{\mathcal{G}}) = 4$ heads and $L = 2$ blocks (one feature block and the score
read-out), giving $25\,796$ parameters for GT and $42\,692$ for GCAT; the
$17$k difference is the second filter tap in $Q$ and $K$ together with
the LayerNorms on their outputs.

\subsubsection{DiGress (small)}
\label{app:arch_details:digress}

The DiGress baseline is the upstream XEy graph transformer
of~\citet{vignac2022digress} at reduced width, which we call
\textsc{digress\_small}: a stack of $L = 4$ blocks that jointly update
node features $X$, edge features $E$, and a global vector $y$, with
FiLM-style cross-modulation between the three streams. Hidden widths are
$d_X = 128$, $d_E = 32$, $d_y = 32$ with $4$ attention heads,
feed-forward widths $(d_{\mathrm{ff}X}, d_{\mathrm{ff}E}, d_{\mathrm{ff}y}) = (128, 32, 128)$
and input/output MLP widths $(64, 32, 64)$, for $678\,468$ parameters.
The attention block is upstream's: the per-pair tensor
$Y = (Q \odot K / \sqrt{d_f})\,(E_1 + 1) + E_2$ is FiLM-modulated by the
edge stream, softmaxed per channel over the key axis, and also written
back as the new edge features; $y$ modulates both $X$ and $E$ and is
updated from pooled node and edge summaries.

\paragraph{Inputs.} To keep the comparison to GT and GCAT about the
attention mechanism rather than the feature set, the baseline receives
the same information they do. Its node input is a constant two-class
one-hot placeholder (the graphs are unattributed) concatenated with the
same $k = 16$ largest-magnitude eigenvectors of the noisy adjacency; its
edge input is the two-class one-hot noisy adjacency; and $y$ holds the
normalised timestep $t/T$. DiGress's hand-crafted structural
augmentation (cycle counts, connected-component indicators, Laplacian
eigenvalues) is switched off. The $Q$/$K$ projections are the standard
linear ones.

\subsubsection{Capacity-matched variants}
\label{app:capacity}

GCAT and \textsc{digress\_small} differ by a factor of $16$ in parameter
count. To separate the effect of the architecture from that of
capacity, we train two further arms in which each model is scaled to
the other's size (Table~\ref{tab:capacity}). For GCAT the width $d$ is
raised from $64$ to $256$, which scales every projection and the
filter taps; for \textsc{digress\_small} $d_X$ is lowered from $128$ to
$32$ with $d_E$, $d_y$, the feed-forward and the input/output MLP widths
scaled in the preset's proportions (a quarter, a quarter, equal, half).
The wide GCAT uses the symmetrically normalised filter operator
described above; everything else, including the training protocol, is
unchanged.

\begin{table}[h]
\centering
\small
\caption{The four arms used in the cross-model comparisons. Parameter
  counts are exact and dataset-independent.}
\label{tab:capacity}
\resizebox{\textwidth}{!}{%
\begin{tabular}{llrl}
\toprule
\textbf{Arm} & \textbf{Model} & \textbf{Parameters} & \textbf{Change from default} \\
\midrule
GT             & spectral, standard $Q/K$      & $25\,796$  & --- \\
GCAT           & spectral, convolutional $Q/K$ & $42\,692$  & --- \\
GCAT-256       & spectral, convolutional $Q/K$ & $662\,276$ & $d = 256$, $\bar A = D^{-1/2} A D^{-1/2}$ \\
\textsc{digress\_small}     & DiGress, standard $Q/K$ & $678\,468$ & --- \\
\textsc{digress\_small}-32  & DiGress, standard $Q/K$ & $44\,820$  & $(d_X, d_E, d_y) = (32, 8, 8)$ \\
\bottomrule
\end{tabular}%
}
\end{table}

\subsection{Training protocol}
\label{app:training}

All generative runs share one protocol (Table~\ref{tab:modattn_shared});
the two model families differ only in the diffusion horizon $T$ and in
the accumulation needed to fit DiGress's edge tensors into memory on the
three large-graph datasets. Each run trains for $10\,000$ optimiser
steps at $16$ graphs per step, which is the same amount of data on every
dataset. Validation loss is measured every $1\,000$ steps on the full
validation split with one $(t, \text{noise})$ draw per graph, a
checkpoint is written at every validation, and the checkpoint with the
lowest validation loss is the one every generative evaluation samples
from. GT and GCAT are trained with three seeds ($7$, $17$, $27$);
\textsc{digress\_small} with seed $7$; the two capacity-matched arms
with seeds $7$ and $17$. All runs use bf16 mixed precision on a single
GPU.

\begin{table}[h]
\centering
\caption{Hyperparameters shared across all generative runs. The
  diffusion horizon is the one setting that differs between the two
  model families; loss and sample comparisons across families are
  therefore made on the normalised axis $t/T$.}
\label{tab:modattn_shared}
\small
\resizebox{\textwidth}{!}{%
\begin{tabular}{lll}
\toprule
\textbf{Block} & \textbf{Setting} & \textbf{Value} \\
\midrule
Architecture
 & GT / GCAT (App.~\ref{app: arch details GT GCAT})   & $k = 16$, $d = 64$, $\text{\normalfont{SD}}(\mathcal{D}_{\mathcal{G}}) = 4$, $L = 2$, $P = 2$ \\
 & \textsc{digress\_small} (App.~\ref{app:arch_details:digress}) & $L = 4$, $(d_X, d_E, d_y) = (128, 32, 32)$, $4$ heads \\
 & dropout                           & $0.0$ \\
 & timestep conditioning             & GT / GCAT: none; DiGress: $t/T$ appended to $y$ \\
\midrule
Diffusion
 & forward process                   & categorical (two edge classes) \\
 & noise schedule                    & cosine-iDDPM \citep{nichol2021improved} \\
 & timesteps $T$                     & $500$ (GT / GCAT), $1\,000$ (DiGress) \\
 & limit distribution                & empirical edge marginal of the training split \\
 & training timesteps                & $t \sim \mathcal{U}\{0, \dots, T\}$ per graph \\
 & sampler                           & ancestral \\
\midrule
Optimisation
 & optimiser                         & AdamW with AMSGrad \\
 & learning rate                     & $2 \times 10^{-4}$, constant \\
 & weight decay                      & $10^{-12}$ \\
 & gradient clipping / EMA           & none / none \\
 & loss                              & masked categorical cross-entropy, edge weight $\lambda_E = 5$ \\
 & precision                         & bf16 mixed \\
\midrule
Schedule
 & optimiser steps                   & $10\,000$ \\
 & graphs per optimiser step         & $16$ \\
 & gradient accumulation (DiGress)   & $4 \times 4$ on PROTEINS and COLLAB, $2 \times 8$ on DEEZER \\
 & validation cadence                & every $1\,000$ steps, full validation split \\
 & checkpoint selection              & lowest validation loss \\
\midrule
Data
 & split                             & $70/10/20$, seed $42$ \\
 & padding                           & to the dataset's $n_{\max}$ (Table~\ref{tab:dataset_stats}) \\
\bottomrule
\end{tabular}%
}
\end{table}

\paragraph{Memory.} DiGress keeps a hidden feature vector per node pair,
so a training step's memory grows with the square of the largest graph
in the batch. On a $24$\,GB GPU \textsc{digress\_small} at $16$ graphs per
batch peaks at $9.5$\,GB when padded to $200$ nodes and runs out of memory
at $350$; PROTEINS, COLLAB and DEEZER contain training graphs of $504$,
$445$ and $363$ nodes. On those three datasets the DiGress runs therefore
accumulate gradients over micro-batches ($4 \times 4$, $4 \times 4$ and
$2 \times 8$ graphs), which keeps $16$ graphs per optimiser step and
the same data and splits as every other cell. GT and GCAT carry only
per-node features and need under $1$\,GB on every dataset.

\subsection{Evaluation protocol for the improvement heatmaps}
\label{app:eval_protocol}

The heatmaps of Section~\ref{sec:exp_new} compare a \emph{treatment}
arm against a \emph{base} arm on the same dataset: GCAT against GT
(within the spectral family, three seeds), and GCAT-256 against
\textsc{digress\_small} (across families, seed $7$). Every cell is a
normalised MMD$^2$ improvement, and a separate row reports the relative
change in validation loss. Both are computed as follows.

\paragraph{Regenerated samples.} For every run we reload the
best-validation checkpoint and draw $24$ graphs in two batches of $12$
under fixed generation seeds, with node counts sampled from the training
split's size distribution. These regenerated graphs, not the $12$
sampled during training, are what the heatmaps score.

\paragraph{Statistics and kernels.} We compute the SPECTRE-style
MMD$^2$ with a Gaussian kernel on the total-variation distance between
per-graph descriptors: degree histograms, clustering-coefficient
histograms, the histogram of normalised-Laplacian eigenvalues (``spectral''),
and ORCA orbit counts (``orbit'').
The heatmaps use one bandwidth per dataset and metric, the
\emph{reference-median} bandwidth $\sigma_{\mathrm{ref}}$: the median
pairwise TV distance among the descriptors of the validation pool.
This replaces the fixed bandwidths of the sweep's in-training evaluation
(degree $1$, clustering $0.1$, spectral $1$, orbit $30$), whose scale is
arbitrary relative to a dataset's own spread. On the five datasets the
two choices agree in sign on $16$ of the $17$ cells they both define in
the GT--GCAT heatmap and on $15$ of $17$ in the cross-model one. The
exception common to both is the IMDB-BINARY orbit cell, which is clearly
positive at the fixed bandwidth and indistinguishable from zero at the
reference one; the other is a PROTEINS orbit cell within $\pm 0.01$ of
zero at either bandwidth.

\paragraph{Two nulls from one protocol.} An MMD$^2$ value on its own does not
say whether a model is close to the data, because the statistic has a
finite-sample floor that depends on the dataset and the metric. We
therefore score every model against two nulls drawn under exactly the
same protocol. In each of $25$ repeats we draw $12$ validation graphs as
the reference set $A$, $12$ disjoint validation graphs $B$, and a set
$E$ of Erd\H{o}s--R\'enyi graphs matched one-for-one to $B$ in node count
and edge density ($\text{\normalfont{SD}}(\mathcal{D}_{\mathcal{G}})(n, p)$ with $p = m / \binom{n}{2}$), and take $12$
of the model's $24$ regenerated graphs. $\mathrm{MMD}^2(A, B)$ is the
\emph{real--real} reference (two samples from the data), $\mathrm{MMD}^2(A, E)$ is the
\emph{density-matched ER} reference (a sample that has the right size and
density and nothing else), and $\mathrm{MMD}^2(A, \text{model})$ is the model's
score. Each is the median over the $25$ repeats, and the reference,
ER and model draws use separate random streams so that adding a model
never moves another model's score.

\paragraph{Cells.} Writing $\mathrm{rr}$ and $\mathrm{er}$ for the two
reference medians and $m(M)$ for the median MMD$^2$ of model $M$, we first
score each model by
\begin{equation}
  s_{\mathrm{norm}}(M) = \frac{m(M) - \mathrm{rr}}{\mathrm{er} - \mathrm{rr}}.
\end{equation}
The real--real reference gives a score of zero, and the density-matched
ER reference gives a score of one. A model can score outside this range.
The denominator scales each model's distance from the real--real reference
by the gap that this metric resolves between the two references. The
heatmap cell is the difference between the base and treatment scores:
\begin{equation}
  \Delta_{\mathrm{norm}} = s_{\mathrm{norm}}(\text{base}) - s_{\mathrm{norm}}(\text{treat})
  = \frac{m(\text{base}) - m(\text{treat})}{\mathrm{er} - \mathrm{rr}}.
\end{equation}
Thus $+1$ means the treatment moved the model one full ``ER gap'' closer
to the data than the base, and green cells favour the treatment. For the
three-seed comparison we average the per-seed model scores before
differencing. A cell is marked N/A when the ER median
does not exceed the $90$th percentile of the real--real draws, in which
case the statistic cannot distinguish the dataset from density-matched
noise and has no power to rank models on it. Colours are clipped at
$\pm 2$; the printed values are exact.

\paragraph{Loss rows.} Within the spectral family (GCAT against GT) the
loss row reports the validation loss of the final checkpoint as a
fraction of the base arm, $(\mathcal{L}_{\text{GT}} - \mathcal{L}_{\text{GCAT}}) / \mathcal{L}_{\text{GT}}$,
averaged over the three seeds. A single validation pass draws one
$(t, \text{noise})$ pair per graph and has a standard deviation of about
$0.05$ (spectral) to $0.10$ (DiGress) in the loss, which is comparable to
some of the effects; the row therefore re-evaluates each checkpoint over
the whole validation split $\lceil 3\,000 / n_{\text{val}} \rceil$ times
(at most $50$) with a fresh seed per pass, for $\approx 3\,000$
$(\text{graph}, t, \text{noise})$ draws per run and a standard error of
about $0.02$.
Across families the total loss is not comparable, because the node term
is identically zero for DiGress (its node class is constant) while the
families also use different horizons. The cross-model loss row therefore
reports the \emph{edge} cross-entropy alone, measured per decile of
$t/T$ with every validation graph's $t$ drawn uniformly inside the
decile and at least three passes per decile ($600$--$2\,900$ draws), and
averages the ten decile means; the row shows
$(\mathcal{L}^{E}_{\text{DiGress}} - \mathcal{L}^{E}_{\text{GCAT-256}}) / \mathcal{L}^{E}_{\text{DiGress}}$.
Note that the MMD$^2$ cells score the best-validation checkpoint while the
loss rows read the final (step $10\,000$) checkpoint.

\subsection{Sampling cost}
\label{app:sampling_cost}

Table~\ref{tab:sampling-time-670k} reports the wall-clock cost of
sampling at matched capacity. Per reverse step GCAT-256 is $1.7$--$1.8\times$
faster than \textsc{digress\_small} on ENZYMES, PROTEINS and IMDB-B, and
$4$--$8\times$ faster on DEEZER and COLLAB. The gap tracks the largest graph in a batch rather than the mean,
because DiGress carries an $n \times n \times d_E$ hidden edge tensor
through every block and a batch is padded to its largest member, whereas
GT and GCAT carry per-node features only. Width changes neither picture
by much: scaling GCAT from $43$k to $662$k parameters adds $10$--$40\%$
to its sampling time, and shrinking DiGress from $678$k to $45$k removes
$20$--$45\%$, because the $n^2$ tensors remain.

\begin{table}[t]
\centering
\small
\caption{Sampling cost at matched capacity ($\approx$670k parameters). Each model
  generates 12 graphs on one A5000 during validation; we report
  \emph{wall}, the end-to-end wall-clock time for that batch, and
  \emph{per-step}, the same time divided by the number of reverse diffusion
  steps ($T=500$ for GCAT, $T=1000$ for
  \textsc{digress\_small}), which removes the differing sampling horizons and
  isolates the per-step cost of one denoiser call. Ratios are
  DiGress$\,/\,$ GCAT, so values above $1$ favour GCAT.
  Times are medians over validation cycles, averaged
  over seeds (GCAT-256: seeds 7 and 17; DiGress: seed 7).}
\label{tab:sampling-time-670k}
\resizebox{\textwidth}{!}{%
\begin{tabular}{lrrrrrrrr}
\toprule
& \multicolumn{2}{c}{nodes (val.)}
& \multicolumn{2}{c}{GCAT-256 (662k)}
& \multicolumn{2}{c}{\textsc{digress\_small} (678k)}
& \multicolumn{2}{c}{ratio} \\
\cmidrule(lr){2-3}\cmidrule(lr){4-5}\cmidrule(lr){6-7}\cmidrule(lr){8-9}
Dataset & mean & max & wall (s) & per-step (ms) & wall (s) & per-step (ms) & wall & per-step \\
\midrule
ENZYMES     &  27.4 &  54 &  7.0 & 14.1 &  23.7 &  23.7 &  3.4$\times$ & 1.7$\times$ \\
PROTEINS    &  38.5 & 328 & 10.0 & 19.9 &  35.2 &  35.2 &  3.5$\times$ & 1.8$\times$ \\
IMDB-B      &  19.0 &  55 &  5.6 & 11.2 &  18.9 &  18.9 &  3.4$\times$ & 1.7$\times$ \\
COLLAB      &  74.2 & 444 & 20.0 & 39.9 & 330.5 & 330.5 & 16.5$\times$ & 8.3$\times$ \\
DEEZER      &  23.2 & 218 &  6.6 & 13.2 &  54.8 &  54.8 &  8.3$\times$ & 4.1$\times$ \\
\bottomrule
\end{tabular}%
}
\end{table}

\subsection{Spectral diversity metric}
\label{app:knn}

The $x$-axis of Figure~\ref{fig:gain vs ghat five} is a
non-parametric estimate $\widehat{\text{\normalfont{SD}}}(\mathcal{D}_{\mathcal{G}})$ of the improvement gap of
Theorem~\ref{thm: loss improvement}, computed on the validation split of
each dataset and normalised per graph so that it is measured per
adjacency entry, the unit of the validation loss on the $y$-axis. The
estimand, the per-graph normalisation, the plug-in estimator, the
frame-alignment step that makes the cross-graph average meaningful, the
noise model, the bias floor of the plug-in and the permutation null that
measures it are detailed below, followed by the hyperparameters of the
figure.

\paragraph{Population estimand.} For each graph $i$ in a dataset with
$n_i$ nodes, let $A_i$ be its clean adjacency, $\tilde A_i$ a noisy
observation, and $(\hat V_{k,i}, \tilde \Lambda_{k,i})$ the top-$k$
eigenpairs of $\tilde A_i$ (the $k$ largest eigenvalues and their
eigenvectors). Define
\begin{equation}
  B_i \;=\; \frac{1}{n_i}\, \hat V_{k,i}^\top A_i \hat V_{k,i} \;\in\; \mathbb{R}^{k \times k},
  \label{eq:B-normalised}
\end{equation}
the clean adjacency expressed in the noisy top-$k$ eigenbasis, scaled by
the graph's own size (the reason for the factor $1/n_i$ is given next).
Theorem~\ref{thm: loss improvement} identifies the gap between the
optimal attention-score losses with
$\text{\normalfont{SD}}(\mathcal{D}_{\mathcal{G}}) = \mathrm{tr}(\mathrm{Cov}\,\mathbb{E}[B \mid \tilde\Lambda_k]) =
\mathbb{E}\|\mathbb{E}[B \mid \tilde\Lambda_k] - \mathbb{E}[B]\|_F^2$, the
variance across graphs of the optimal spectral predictor
$f^\star(\tilde\Lambda_k) = \mathbb{E}[B \mid \tilde\Lambda_k]$. Its ratio
to the total variance $\mathrm{tr}(\mathrm{Cov}\, B) =
\mathbb{E}\|B - \mathbb{E}[B]\|_F^2$ is the fraction of variance explained
by the noisy spectrum, the $R^2$ of $f^\star$
(Propositions~\ref{prop: minimum lin loss}--\ref{prop: minimum filter loss});
the theorem states the loss gap itself in terms of $\text{\normalfont{SD}}(\mathcal{D}_{\mathcal{G}})$, and $\text{\normalfont{SD}}(\mathcal{D}_{\mathcal{G}})$ is what
the figure plots.

\paragraph{Normalisation by graph size.} Without the factor $1/n_i$,
because $\hat V_{k,i}$ has orthonormal columns,
\begin{equation}
  \|\hat V_{k,i}^\top A_i \hat V_{k,i}\|_F^2
  \;=\; \|\hat V_{k,i}\hat V_{k,i}^\top A_i \hat V_{k,i}\hat V_{k,i}^\top\|_F^2,
\end{equation}
so $\text{\normalfont{SD}}(\mathcal{D}_{\mathcal{G}})$ is a squared error on the rank-$k$ part of the $n_i \times n_i$
adjacency, summed over all $n_i^2$ entries of the graph. The validation loss it is compared
against is a cross-entropy averaged over the node pairs of each graph and
then over graphs, so the matching quantity is the per-graph gap divided by
$n_i^2$, averaged over graphs. Scaling $B_i$ by $1/n_i$ before any
averaging achieves exactly this, since every term of $\widehat{\text{\normalfont{SD}}}(\mathcal{D}_{\mathcal{G}})$ below is
quadratic in $B$. A single dataset-level divisor such as $n_{\max}^2$
does not: graph sizes within a dataset vary by up to a factor of eight
(PROTEINS and DEEZER-EGO-NETS span $10$ to $84$ nodes), the mean size is
between $19$ and $58$ against band maxima of $54$ to $124$
(Table~\ref{tab:ghat_settings}), and dividing by $n_{\max}^2$ instead
understates $\widehat{\text{\normalfont{SD}}}(\mathcal{D}_{\mathcal{G}})$ by a factor of $6$ to $20$ depending on the dataset.
The conditioning
vector $\tilde\Lambda_{k,i}$ is left on its natural scale; the
normalisation acts on the regression target only.

\paragraph{Plug-in estimator.} The population quantities are inaccessible
and each is replaced by a finite-sample plug-in. Given $N$ graphs with
paired observations $\{(B_i, \tilde\Lambda_{k,i})\}_{i=1}^{N}$, the
optimal predictor is estimated by leave-one-out kNN regression on the
noisy eigenvalues,
\begin{equation}
\hat B_i \;=\; \frac{1}{m} \sum_{j \in \mathcal{N}_m(i)} B_j,
\end{equation}
where $\mathcal{N}_m(i)$ is the set of $m$ graphs other than $i$ whose
sorted eigenvalue vectors $\tilde\Lambda_{k,j}$ are closest to
$\tilde\Lambda_{k,i}$ in Euclidean distance. The plug-in estimates of $\text{\normalfont{SD}}(\mathcal{D}_{\mathcal{G}})$
and of the total variance use the same sample,
\begin{equation}
\widehat{\text{\normalfont{SD}}}(\mathcal{D}_{\mathcal{G}}) \;=\; \frac{1}{N} \sum_{i=1}^{N} \|\hat B_i - \bar B\|_F^2,
\qquad
\widehat{\mathrm{tr}(\mathrm{Cov}\, B)} \;=\; \frac{1}{N} \sum_{i=1}^{N} \|B_i - \bar B\|_F^2,
\label{eq:ghat}
\end{equation}
with $\bar B = \tfrac{1}{N} \sum_i B_i$. The estimator is non-parametric:
no analytic form for $f^\star$ is assumed. $\widehat{\text{\normalfont{SD}}}(\mathcal{D}_{\mathcal{G}})$ is the reported
quantity; the ratio $\widehat{\text{\normalfont{SD}}}(\mathcal{D}_{\mathcal{G}}) / \widehat{\mathrm{tr}(\mathrm{Cov}\, B)}$
is recorded beside it but, being scale-free, does not carry the
per-entry information the comparison with a loss requires.

\paragraph{Frame alignment.} The matrix $B_i$ is not invariant to the
gauge of $\hat V_{k,i}$: replacing $\hat V_{k,i}$ by $\hat V_{k,i} Q_i$ for
any $Q_i \in O(k)$ maps $B_i \to Q_i^\top B_i Q_i$. Averaging matrices
across different per-graph frames adds incommensurable quantities and
inflates $\widehat{\text{\normalfont{SD}}}(\mathcal{D}_{\mathcal{G}})$, because $Q_i$ is partly predictable from
$\tilde\Lambda_{k,i}$ (both derive from the same noisy decomposition). We
therefore align every $\hat V_{k,i}$ by orthogonal Procrustes (reflections
allowed) against a single dataset-wide reference frame, the extrinsic
Grassmannian mean of the clean top-$k$ eigenbases (the top-$k$ left
singular vectors of the $n \times Nk$ matrix formed by stacking the clean
eigenvector blocks side by side). With one common frame, the residual
rotation cannot co-vary with $\tilde\Lambda_k$, and the convention
artefact that would otherwise feed a spurious signal into the estimator
is closed off.

\paragraph{Noise model.} Noise is the DiGress corruption with
$\bar\alpha = 1 - \varepsilon$: every edge of the upper triangle is
flipped independently with probability $\varepsilon/2$ and mirrored, so
the limit $\varepsilon = 1$ is a uniformly random graph. For the figure,
$\varepsilon$ is not fixed but drawn from the diffusion schedule of the
generative runs: each graph draws its own timestep
$t_i \sim \mathcal{U}\{1, \dots, T\}$ and is corrupted at
$\varepsilon_i = 1 - \bar\alpha_{t_i}$ of the cosine-iDDPM schedule
\citep{nichol2021improved}, so that $\widehat{\text{\normalfont{SD}}}(\mathcal{D}_{\mathcal{G}})$ is estimated under the same
noise marginal that the runs' validation loss averages over
(Table~\ref{tab:modattn_shared}; the schedule is the runs' own, with
$T = 500$, and the mean drawn level is $\approx 0.52$).
Two differences from the runs' forward process remain: the runs' limit
distribution is the empirical edge marginal rather than uniform, and GT
and GCAT are not conditioned on $t$, which is why $\tilde\Lambda_{k,i}$
alone, without $\varepsilon_i$, is the conditioning vector. Fixed levels
$\varepsilon \in \{0.05, 0.1, 0.2\}$ were computed alongside: the five
datasets rank identically at every level and every $k$, and the Pearson
correlation with the gain stays between $+0.77$ and $+0.82$ across all
twelve $(\varepsilon, k)$ combinations, with $\rho = +0.60$ throughout.

\paragraph{Padding.} Every graph of a dataset is zero-padded to the
maximum of its size band so that the eigenvector blocks share an ambient
dimension and the Fréchet reference is well defined. Corruption is
confined to each graph's real $n_i \times n_i$ block; the padding rows
receive no noise, as in training, where the node mask is re-imposed after
the forward process.

\paragraph{Bias floor.} Under the null $B \perp \tilde\Lambda_k$, the
neighbours $\mathcal{N}_m(i)$ are effectively a random subsample of the
$B_j$, so $\mathrm{Var}(\hat B_i) \approx \mathrm{Var}(B)/m$ and $\widehat{\text{\normalfont{SD}}}(\mathcal{D}_{\mathcal{G}})$
acquires a floor
\begin{equation}
\mathbb{E}\bigl[\widehat{\text{\normalfont{SD}}}(\mathcal{D}_{\mathcal{G}}) \;\big|\; B \perp \tilde\Lambda_k\bigr]
\;\approx\; \Bigl(\frac{1}{m} - \frac{1}{N-1}\Bigr)\,
\widehat{\mathrm{tr}(\mathrm{Cov}\, B)},
\end{equation}
which for $(m, N) = (10, 200)$ is $0.095$ of the total variance.
Comparing $\widehat{\text{\normalfont{SD}}}(\mathcal{D}_{\mathcal{G}})$ directly to zero is therefore uninformative.

\paragraph{Permutation null.} We measure the floor on each dataset by
shuffling the conditioning vectors with a uniformly random permutation
$\pi$ and recomputing the same plug-in,
$\widehat{\text{\normalfont{SD}}}(\mathcal{D}_{\mathcal{G}})_{\text{null}} = \widehat{\text{\normalfont{SD}}}(\mathcal{D}_{\mathcal{G}})\bigl(\{(B_i, \tilde\Lambda_{k,\pi(i)})\}\bigr)$.
The shuffle preserves both marginals and breaks only the joint, so its
expectation matches the kNN bias to leading order; the measured nulls
are within $16\%$ of the formula above on every dataset.
Table~\ref{tab:ghat_settings} reports $\widehat{\text{\normalfont{SD}}}(\mathcal{D}_{\mathcal{G}})_{\text{null}}$ beside
$\widehat{\text{\normalfont{SD}}}(\mathcal{D}_{\mathcal{G}})$. It is $20$--$40\%$ of $\widehat{\text{\normalfont{SD}}}(\mathcal{D}_{\mathcal{G}})$ on four of the five datasets and
$76\%$ on DEEZER-EGO-NETS, whose $\widehat{\text{\normalfont{SD}}}(\mathcal{D}_{\mathcal{G}})$ is therefore the one value close
to the estimator's resolution. The figure plots the uncalibrated $\widehat{\text{\normalfont{SD}}}(\mathcal{D}_{\mathcal{G}})$;
the calibrated margin $\widehat{\text{\normalfont{SD}}}(\mathcal{D}_{\mathcal{G}}) - \widehat{\text{\normalfont{SD}}}(\mathcal{D}_{\mathcal{G}})_{\text{null}}$ gives $r = +0.71$ in
place of $+0.81$ with the same ordering.

\paragraph{Hyperparameters.} The figure uses $k = 8$ and $m = 10$
neighbours; $k \in \{4, 16\}$ was computed alongside and changes neither
the ranking nor, beyond the range quoted above, the correlation. Graphs
are the validation split of each
dataset (the split every validation loss is measured on), restricted to
the size band of Table~\ref{tab:ghat_settings} and capped at $N = 200$ by
a seeded random subsample, because the bias floor depends on $N$ and an
uncapped comparison across datasets would compare estimator bias as much
as spectra. Three seeds ($42$, $123$, $2\,024$) control the subsample,
the corruption draws and the per-graph timesteps; each point of the
figure is the mean over the three seeds and its horizontal error bar
their standard deviation. Eigendecompositions run in float64.

\paragraph{The gain axis.} The $y$-axis is the mean over the three
training seeds ($7$, $17$, $27$) of the difference in validation edge
cross-entropy between GT and GCAT after $10\,000$ optimiser steps, on the
full validation split with one $(t, \text{noise})$ draw per graph
(Appendix~\ref{app:training}); the vertical error bar is the standard
deviation over seeds. This gain is the difference between two trained
models, whereas $\text{\normalfont{SD}}(\mathcal{D}_{\mathcal{G}})$ is the gap between the two Bayes-optimal predictors
of Theorem~\ref{thm: loss improvement}; the comparison tests whether the
two co-vary across datasets, not whether they agree in scale.

\begin{table}[h]
  \centering
  \small
  \caption{Per-dataset settings and values behind
    Figure~\ref{fig:gain vs ghat five} ($k = 8$, $m = 10$,
    $\varepsilon$ drawn from the cosine schedule with $T = 500$). $N$ is the number of validation
    graphs after banding and the $200$-graph cap; ``mean $n_i$'' is their
    mean size before padding. $\widehat{\text{\normalfont{SD}}}(\mathcal{D}_{\mathcal{G}})$ and its permutation null are means
    over the three estimation seeds ($\pm$ one standard deviation), the
    gain a mean over the three training seeds. All of $\widehat{\text{\normalfont{SD}}}(\mathcal{D}_{\mathcal{G}})$, its null
    and the gain are per adjacency entry.}
  \label{tab:ghat_settings}
  \resizebox{\textwidth}{!}{%
  \begin{tabular}{@{}lrrrrrr@{}}
    \toprule
    Dataset & band $[n_{\min}, n_{\max}]$ & $N$ & mean $n_i$ & $\widehat{\text{\normalfont{SD}}}(\mathcal{D}_{\mathcal{G}}) \times 10^{2}$ & $\widehat{\text{\normalfont{SD}}}(\mathcal{D}_{\mathcal{G}})_{\text{null}} \times 10^{2}$ & gain (GT $-$ GCAT) \\
    \midrule
    ENZYMES         & $[10, 54]$ & $59$ & $27.8$ & $1.23 \pm 0.05$ & $0.39$ & $0.107 \pm 0.024$ \\
    PROTEINS        & $[10, 84]$ & $89$ & $31.8$ & $1.64 \pm 0.14$ & $0.33$ & $0.057 \pm 0.013$ \\
    IMDB-BINARY     & $[12, 55]$ & $100$ & $19.0$ & $2.94 \pm 0.17$ & $1.17$ & $0.225 \pm 0.031$ \\
    COLLAB          & $[32, 124]$ & $200$ & $57.6$ & $4.83 \pm 0.25$ & $1.13$ & $0.204 \pm 0.031$ \\
    DEEZER-EGO-NETS & $[11, 84]$ & $200$ & $20.2$ & $0.63 \pm 0.02$ & $0.48$ & $0.074 \pm 0.001$ \\
    \bottomrule
  \end{tabular}%
  }
\end{table}

\subsection{Compute resources}\label{app:compute}

\paragraph{Hardware.} Every experiment in this paper was run on a single
workstation with two NVIDIA RTX A5000 GPUs ($24$\,GB GDDR6 each),
$64$ physical / $128$ logical CPU cores and $\approx 270$\,GB host RAM.
Each training run is a single-GPU job; the two cards were used as two
independent workers. The improvement-gap estimation
(Appendix~\ref{app:knn}) is CPU-only; the sweep behind
Table~\ref{tab:ghat_settings} (5 datasets $\times$ 3 seeds $\times$
4 noise levels $\times$ 3 values of $k$ $\times$ real/null
$\times$ raw/per-entry) completes in $\approx 3$\,minutes.

\paragraph{Training cost.} A $10\,000$-step GT or GCAT run takes
$7$--$11$ minutes on the four small-graph datasets and $18$--$22$ minutes
on COLLAB; the wide GCAT-256 adds under $10\%$ to these,
since per-step cost is set by the padded graph size rather than the
width. A \textsc{digress\_small} run takes $19$--$25$ minutes on ENZYMES
and IMDB-BINARY, $41$ minutes on DEEZER, $97$ on PROTEINS and $132$ on
COLLAB, the last three under gradient accumulation; the narrow \textsc{digress\_small}-32 is $20$--$45\%$
faster. The GT and GCAT runs reported here (six datasets, three seeds,
two arms) consumed $\approx 7.4$ GPU-hours, the six \textsc{digress\_small}
runs $\approx 6.6$ GPU-hours, and the $24$ capacity-matched runs
$\approx 11.4$ GPU-hours. Including the synthetic-SBM datasets, the
filter-tap and normalisation ablations and the DiGress attention
ablations that informed the analysis but are not reported, the sweeps
behind this paper total $\approx 120$ GPU-hours. Re-evaluating the
$150$ checkpoints of the loss rows with $\approx 3\,000$ draws each took
$42$ minutes on one GPU.

\clearpage
\section{Proofs}\label{app: proofs}

\subsection{Proof of Proposition~\ref{prop: minimum lin loss}}\label{app: prop minimum lin loss proof}

We begin by expanding the loss using the fact that $\|M\|_F^2 = \mathrm{tr}(M^\top M)$:
\begin{equation}
    \ell(W) = \mathbb{E}\,\mathrm{tr}\!\Big[(\tilde{U} W \tilde{U}^\top)^2 - 2A\,\tilde{U} W \tilde{U}^\top + A^2\Big]
\end{equation}
Further, using $\tilde{U}^\top \tilde{U} = I$ and the cyclic property of trace:
\begin{itemize}
    \item $\mathrm{tr}[(\tilde{U} W \tilde{U}^\top)^2] = \mathrm{tr}[\tilde{U} W^2 \tilde{U}^\top] = \mathrm{tr}[W^2]$
    \item $\mathrm{tr}[A\,\tilde{U} W \tilde{U}^\top] = \mathrm{tr}[W\,\tilde{U}^\top A \tilde{U}]$
\end{itemize}
So the loss reduces to:
\begin{equation}
    \ell(W) = \mathbb{E} \left[ \mathrm{tr}(W^2) - 2\,\mathrm{tr}(W\,B) + \mathrm{tr}(A^2) \right], \quad \text{where } B = \tilde{U}^\top A \tilde{U}
\end{equation}

which we can write as:
\begin{equation}
    \ell(W) = \mathbb{E} \left[ \mathrm{tr}(W^2) - 2\,\mathrm{tr}(W\,B) + \mathrm{tr}(B^2) + \mathrm{tr}(A^2 - B^2) \right]
\end{equation}

\begin{equation}
    \ell(W) = \mathbb{E} \left[ \left\| W - B \right\|_F^2 + \left\| A\right\|_F^2 - \left\| B \right\|_F^2 \right]
\end{equation}

The minimization reduces to:
\begin{equation}
    W^* = \argmin_{W} \,\,\mathbb{E} \left[ \left\| W - B \right\|_F^2 \right] = \mathbb{E}_{A,\,\mathcal{E}} [B]
\end{equation}

Define $c := \mathbb{E} \left[ \left\| A \right\|_F^2 - \left\| B \right\|_F^2 \right]$. Then the minimum achievable loss is:
\begin{equation}
    \ell_{\text{A}} := \ell(W^*) = \mathbb{E} \left[ \left\| \mathbb{E}[B] - B \right\|_F^2  \right] + c = \mathrm{tr}(\mathrm{Cov}(\mathrm{vec}(B))) + c
\end{equation}

\subsection{Proof of Proposition~\ref{prop: minimum filter loss}}\label{app: prop minimum filter loss proof}
\begin{equation}
    \ell(h) = \mathbb{E} \left[ \left\| \tilde{U} h(\tilde{\Lambda}) \tilde{U}^\top - A \right\|_F^2 \right]
\end{equation}

Similarly to the linear case, we can write the loss as:
\begin{equation}
    \ell(h) = \mathbb{E} \left[ \left\| h(\tilde{\Lambda}) - B \right\|_F^2 + \left\| A\right\|_F^2 - \left\| B \right\|_F^2 \right]
\end{equation}

The minimization reduces to:
\begin{equation}
    h^*(\cdot) = \argmin_{g} \,\,\mathbb{E} \left[ \left\| h(\tilde{\Lambda}) - B \right\|_F^2 \right]
\end{equation}

Apply the tower property (law of iterated expectations):
\begin{equation}
    \mathbb{E}\!\left[\|h(\tilde{\Lambda}) - B\|_F^2\right] = \mathbb{E}_{\tilde{\Lambda}}\!\left[\mathbb{E}\!\left[\|h(\tilde{\Lambda}) - B\|_F^2 \;\middle|\; \tilde{\Lambda}\right]\right]
\end{equation}
Since $h(\tilde{\Lambda})$ is deterministic given $\tilde{\Lambda}$, the inner conditional expectation is a pointwise squared-error problem for each realization of $\tilde{\Lambda}$. The minimizer of $\mathbb{E}[\|c - B\|_F^2 \mid \tilde{\Lambda}]$ over a constant $c$ is the conditional mean:
\begin{equation}
    h^*(\tilde{\Lambda}) = \mathbb{E}\!\left[B \;\middle|\; \tilde{\Lambda}\right] = \mathbb{E}\!\left[\tilde{U}^\top A\, \tilde{U} \;\middle|\; \tilde{\Lambda}\right]
\end{equation}

In contrast to the minimizer in the linear attention class which was the unconditional mean $W^*= \mathbb{E}[B]$,

 with spectral attention the minimizer is $h^*(\tilde{\Lambda}) = \mathbb{E}[B \mid \tilde{\Lambda}]$, i.e., the conditional mean given the noisy eigenvalues. Then the minimum achievable loss becomes:
\begin{equation}
    \ell^\star_{\text{SA}} := \ell(h^\star) = \mathbb{E}_{A,\,\mathcal{E}}\!\left[\|B - \mathbb{E}[B \mid \tilde{\Lambda}]\|_F^2\right] + c
\end{equation}

\begin{equation}
    \Rightarrow \ell^\star_{\text{SA}} = \mathbb{E}_{\tilde{\Lambda}} \left[ \underbrace{\mathbb{E}_{B | \tilde{\Lambda}}\!\left[\|B - \mathbb{E}[B \mid \tilde{\Lambda}]\|_F^2\right]}_{\mathrm{tr}\left(\mathrm{Cov}(\mathrm{vec}(B) \mid \tilde{\Lambda})\right)} \right] + c
\end{equation}

\subsection{Proof of Theorem~\ref{thm: loss improvement}}\label{app: proof loss improvement}

By the law of total variance:
\begin{equation}
    \mathrm{Cov}(\mathrm{vec}(B)) = \mathbb{E}_{\tilde{\Lambda}} \left[ \mathrm{Cov}(\mathrm{vec}(B) \mid \tilde{\Lambda}) \right] + \mathrm{Cov}(\mathbb{E}[\mathrm{vec}(B) \mid \tilde{\Lambda}])
\end{equation}

Taking the trace of both sides, we have:
\begin{equation}
    \ell^\star_{\text{A}} = \ell^\star_{\text{SA}} + \mathrm{tr}(\mathrm{Cov}(\mathbb{E}[\mathrm{vec}(B) \mid \tilde{\Lambda}]))
\end{equation}

Thus, the improvement is:

\begin{equation}
   \ell^\star_{\text{A}} - \ell^\star_{\text{SA}} = \mathrm{tr}(\mathrm{Cov}(\mathbb{E}[\mathrm{vec}(B) \mid \tilde{\Lambda}]))
\end{equation}
which can also be written as:
\begin{equation}
    \ell^\star_{\text{A}} - \ell^\star_{\text{SA}} = \mathbb{E}\left[\left\|\mathbb{E}[B \mid \tilde{\Lambda}] - \mathbb{E}[B]\right\|_F^2\right]
\end{equation}

\subsection{Proof of Lemma~\ref{lem:1}} \label{app:lem1}
We start by setting
\[
\sum_{l=0}^{K-1}\tilde{\Lambda}^l H_Q^{(l)}=\mathbf{I}.
\]
This choice is feasible, for instance by taking $H_Q^{(0)}=\mathbf{I}$ and
$H_Q^{(l)}=\mathbf{0}$ for all $l\ge 1$. Next, choose each $H_K^{(l)}$ to be a
constant diagonal matrix, i.e., $H_K^{(l)}=h_k^{(l)}\mathbf{I}$. Under these
choices, Eq.~\eqref{eq:productGF} reduces to the spectral filter
\[
    f(U)
    =
    U
    \left(
    \sum_{l=0}^{L-1} h_k^{(l)} \Lambda^l
    \right)^\top
    U^\top .
\]
Since $\Lambda$ is diagonal, the matrix
\[
    \sum_{l=0}^{L-1} h_k^{(l)} \Lambda^l
\]
acts independently on each eigenvalue. Equivalently, it implements the
pointwise polynomial filter
\[
    w(\lambda)
    =
    \sum_{l=0}^{L-1} h_k^{(l)} \lambda^l .
\]

It remains to show that such a polynomial filter can approximate
$\eta_{\mathrm{out}}$. Although $\eta_{\mathrm{out}}$ is not analytic at the
thresholds $\lambda=\pm 2v$, it is continuous on any compact interval
$[a,b]\subset\mathbb{R}$. Therefore, by the Weierstrass approximation theorem
\citep{rudin1976principles}, for every $\varepsilon>0$ there exists a real
polynomial $P$ such that
\[
    \sup_{\lambda\in[a,b]}
    \left|
    P(\lambda)-\eta_{\mathrm{out}}(\lambda)
    \right|
    <\varepsilon .
\]
Writing
\[
    P(\lambda)=\sum_{l=0}^{L-1} h_k^{(l)}\lambda^l
\]
for sufficiently large $L$, we obtain a choice of coefficients
$\{h_k^{(l)}\}_{l=0}^{L-1}$ such that the corresponding graph filter
approximates $\eta_{\mathrm{out}}$ uniformly on $[a,b]$. This completes the
proof.

\subsection{Proof of Theorem~\ref{thm: low variance sbm shrinkage}}\label{app: low variance sbm shrinkage proof}
\begin{proof}
    We first show that for a given graph $A$, $B := \tilde{U}^\top A \tilde{U}$ tends to a deterministic matrix as $n \to \infty$. We begin by writing $B$ as:

\begin{equation}
    B = \tilde{U}^\top A \tilde{U} = \tilde{U}^\top U \Lambda U^\top \tilde{U} 
\end{equation}

Consider that $A$ is low rank therefore we can write:
\begin{equation}
    A = \sum_{i = 1}^r \lambda_i u_i u_i^\top
\end{equation}

Thus we have:
\begin{equation}
    B_{ij} = \sum_{l = 1}^r \lambda_l (\tilde{u}_i^\top u_l)(\tilde{u}_j^\top u_l)
\end{equation}

According to Lemma \ref{lem: leading-order overlap} we have:
\begin{equation}
    \lim_{n \to \infty} B_{ii} = \lim_{n \to \infty} \sum_{l = 1}^r \lambda_l |\tilde{u}_i^\top u_l|^2 = \lambda_i \lim_{n \to \infty} |\tilde{u}_i^\top u_i|^2  + \sum_{l = 1, l \neq i}^r \lambda_l \lim_{n \to \infty} |\tilde{u}_i^\top u_l|^2
\end{equation}
The sum over $l \neq i$ tends to zero almost surely, thus we have:
\begin{equation}
    \lim_{n \to \infty} B_{ii} = \lambda_i \alpha(\lambda_i) =  \lambda_i (1 - \frac{v^2}{\lambda_i^2}) = \lambda_i - \frac{v^2}{\lambda_i}
\end{equation}

for the off-diagonal elements we have:
\begin{equation}
    \lim_{n \to \infty} B_{ij} = \lim_{n \to \infty} \sum_{l = 1}^r \lambda_l (\tilde{u}_i^\top u_l)(\tilde{u}_j^\top u_l) = 0,
\end{equation}
where since $i \neq j$, one of the inner products will almost surely tend to zero regardless of the value of $l$.

So, in summary we have for a given fixed graph $A$, $B$ tends to a deterministic matrix as $n \to \infty$:
\begin{equation}
    \lim_{n \to \infty} B = \alpha(\Lambda)
\end{equation}

Furthermore, from Lemma \ref{lem: leading-order overlap} we have the deterministic relationship
between the noisy and clean eigenvalues:
\begin{equation}
    \rho(\lambda_i) = \tilde{\lambda}_i = \lambda_i + \frac{v^2}{\lambda_i}
\end{equation}
Thus we have:
\begin{equation}
    \lambda_i = \frac{\tilde{\lambda}_i + \sqrt{\tilde{\lambda}_i^2 - 4 v^2}}{2}
\end{equation}
Finally we can write:
\begin{equation}
    \alpha(\lambda_i)  = \alpha\left(\frac{\tilde{\lambda}_i + \sqrt{\tilde{\lambda}_i^2 - 4 v^2}}{2}\right) = 1 - \frac{v^2}{\left(\frac{\tilde{\lambda}_i + \sqrt{\tilde{\lambda}_i^2 - 4 v^2}}{2}\right)^2} = \sqrt{\tilde{\lambda}_i^2 - 4 v^2}
\end{equation}

Finally we define the function $\eta_{out}(\cdot)$ as:
\begin{equation}
    \eta_{out}(\tilde{\lambda}) = \begin{cases}
        \sqrt{\tilde{\lambda}^2 - 4 v^2} & \text{if } \tilde{\lambda} > 2v \\
        0 & \text{if } \tilde{\lambda} \leq 2v
    \end{cases}
\end{equation}

Thus in the limit almost surely we have:
\begin{equation}
    \lim_{n \to \infty} B = \eta_{out}(\tilde{\Lambda})
\end{equation}

\end{proof}

The following is a restatement of Theorem 2.1 and Theorem 2.2 
from~\citep{benaychgeorges2010eigenvalueseigenvectorsfinitelow}.

\begin{lemma}[Leading-order overlap for spiked Wigner]
\label{lem: leading-order overlap}
Let $W_n$ be an $n \times n$ real symmetric (or complex Hermitian) 
Wigner matrix with i.i.d.\ entries (up to symmetry) satisfying 
$\mathbb{E}[W_{ij}] = 0$ and $\mathbb{E}[W_{ij}^2] = v^2/n$ for $i \neq j$, 
with finite fourth moment. The empirical spectral distribution of $W_n$ 
converges almost surely to the semicircle law on $[-2v, 2v]$, with 
edge $b = 2v$ and Stieltjes transform
\begin{equation}
g(z) := \int \frac{1}{z-x}\,d\mu_{sc}(x) 
= \frac{z - \sqrt{z^2 - 4v^2}}{2v^2}, 
\qquad z \in \mathbb{R},\ |z| > 2v,
\end{equation}
where the branch of the square root is chosen so that $g(z) \to 0$ as 
$z \to \infty$.

Let $A_n$ be a deterministic (or random, independent of $W_n$) 
symmetric matrix of fixed rank $r$, with spectral decomposition
\begin{equation}
A_n = \sum_{i=1}^{r} \lambda_i\, u_i u_i^\top = U \Lambda U^\top,
\end{equation}
where $\lambda_1 > \lambda_2 > \cdots > \lambda_r > 0$ are fixed (distinct) and 
$u_1, \ldots, u_r \in \mathbb{R}^n$ are orthonormal. 
(Negative spikes are handled symmetrically.)

Define the perturbed matrix $\tilde{A}_n := W_n + A_n$, and 
let $\tilde\lambda_1 \geq \tilde\lambda_2 \geq \cdots \geq \tilde\lambda_n$ be its 
eigenvalues with corresponding orthonormal eigenvectors 
$\tilde{u}_1, \ldots, \tilde{u}_n$.

Let $r_+ \in \{0, 1, \ldots, r\}$ denote the number of supercritical spikes, 
i.e., the number of indices $i$ with $\lambda_i > v$. Then the following 
hold almost surely as $n \to \infty$:

\begin{enumerate}
\item[(i)] \emph{Outlier eigenvalue locations.} For each $i \in \{1, \ldots, r_+\}$,
\begin{equation}
\tilde\lambda_i \xrightarrow{a.s.} \rho(\lambda_i) 
:= \lambda_i + \frac{v^2}{\lambda_i} = g^{-1}(1/\lambda_i),
\end{equation}
and $\rho(\lambda_i) > 2v$. For $i > r_+$, $\tilde\lambda_i \to 2v$.

\item[(ii)] \emph{Eigenvector overlap --- diagonal.} For each 
$i \in \{1, \ldots, r_+\}$,
\begin{equation}
|\langle \tilde{u}_i, u_i\rangle|^2 
\xrightarrow{a.s.} \alpha(\lambda_i) 
:= \frac{-1}{\lambda_i^2\, g'(\rho(\lambda_i))} 
= 1 - \frac{v^2}{\lambda_i^2}.
\end{equation}

\item[(iii)] \emph{Eigenvector overlap --- cross terms.} For 
$i, j \in \{1, \ldots, r_+\}$ with $i \neq j$,
\begin{equation}
\langle \tilde{u}_i, u_j\rangle \xrightarrow{a.s.} 0.
\end{equation}
\end{enumerate}
\end{lemma}

\subsection{Proof of Lemma~\ref{lem: projection helps}}\label{app: proof of projection helps}
\begin{proof}

    We work in the equal block case: $m_\ell = m$ for all $\ell$, $n = km$, $\lambda_{\max} = m$.
    Write $\tilde{U}_k = U_k R + U_\perp S$ where $U_\perp \in \mathbb{R}^{n \times (n-k)}$ collects the eigenvectors
    orthogonal to the principal eigenspace, $R \in \mathbb{R}^{k \times k}$, $S \in \mathbb{R}^{(n-k) \times k}$,
    and $R^\top R + S^\top S = I_k$ (orthonormality of $\tilde{U}_k$).
    The matrix $S$ measures the eigenvector leakage out of the true principal eigenspace.

    We first characterize $\ell_{\text{basis}}$ and $\ell_{\text{proj}}$ in terms of $S$.
    \paragraph{Characterization of $\ell_{\text{basis}}$.}
    Recall that
    \begin{equation}
        \ell_{\text{basis}} = \mathbb{E}_{\mathcal{E}} \left\| \hat{A} - A \right\|_F^2.
    \end{equation}
    In the equal block case $\Lambda_k = mI_k$, so
    \begin{equation}\label{eq:Ahat_minus_A_equal}
        \hat{A} - A = m(\tilde{U}_k\tilde{U}_k^\top - U_k U_k^\top) = m(\tilde{P} - P_U)
    \end{equation}
    where $\tilde{P} = \tilde{U}_k\tilde{U}_k^\top$ is the projector onto the noisy principal eigenspace.
    Substituting $\tilde{U}_k = U_k R + U_\perp S$ and using $U_k^\top U_k = I_k$, $U_k^\top U_\perp = 0$, $U_\perp^\top U_\perp = I_{n-k}$, we expand
    \begin{equation}
        \tilde{U}_k \tilde{U}_k^\top
        = U_k R R^\top U_k^\top + U_k R S^\top U_\perp^\top + U_\perp S R^\top U_k^\top + U_\perp S S^\top U_\perp^\top,
    \end{equation}
    and therefore
    \begin{equation}\label{eq:PtildeP_decomp}
        \tilde{P} - P_U
        = U_k (R R^\top - I_k) U_k^\top + U_k R S^\top U_\perp^\top + U_\perp S R^\top U_k^\top + U_\perp S S^\top U_\perp^\top.
    \end{equation}
    The four terms in~\eqref{eq:PtildeP_decomp} are mutually Frobenius-orthogonal, since their row and column spaces lie in either $\mathrm{range}(U_k)$ or $\mathrm{range}(U_\perp)$, which are orthogonal. Hence
    \begin{equation}\label{eq:PtildeP_norm_split}
        \|\tilde{P} - P_U\|_F^2
        = \|R R^\top - I_k\|_F^2 + 2\|R S^\top\|_F^2 + \|S S^\top\|_F^2.
    \end{equation}
    Using $R^\top R = I_k - S^\top S$ and the cyclic property of the trace,
    \begin{equation}\label{eq:RRt_minus_I_norm}
        \|R R^\top - I_k\|_F^2 = \mathrm{tr}((R^\top R)^2) - 2\,\mathrm{tr}(R^\top R) + k = \|S^\top S\|_F^2,
    \end{equation}
    \begin{equation}
        \|R S^\top\|_F^2 = \mathrm{tr}(R^\top R\, S^\top S) = \mathrm{tr}((I_k - S^\top S)\, S^\top S) = \|S\|_F^2 - \|S^\top S\|_F^2,
    \end{equation}
    Substituting into~\eqref{eq:PtildeP_norm_split} gives
    \begin{equation}
        \|\tilde{P} - P_U\|_F^2 = 2\|S\|_F^2.
    \end{equation}
    Taking expectations,
    \begin{equation}\label{eq:lbasis_equal}
        \ell_{\text{basis}} = m^2\, \mathbb{E}\left[\|\tilde{P} - P_U\|_F^2\right] = 2 m^2\, \mathbb{E}\left[\|S\|_F^2\right]
    \end{equation}

    \paragraph{Characterization of $\ell_{\text{proj}}$.}
    The projected estimator is $P_U \hat{A} P_U^\top$ with $P_U = U_k U_k^\top$. Since $\hat{A} = m\, \tilde{U}_k \tilde{U}_k^\top$ and $U_k^\top \tilde{U}_k = R$,
    \begin{equation}
        P_U \hat{A} P_U^\top
        = m\, U_k (U_k^\top \tilde{U}_k)(\tilde{U}_k^\top U_k) U_k^\top
        = m\, U_k R R^\top U_k^\top.
    \end{equation}
    The clean matrix is $A = m\, U_k U_k^\top = U_k (m I_k) U_k^\top$, so
    \begin{equation}
        P_U \hat{A} P_U^\top - A = m\, U_k (R R^\top - I_k) U_k^\top.
    \end{equation}
    Since $U_k$ has orthonormal columns,
    \begin{equation}
        \| P_U \hat{A} P_U^\top - A \|_F^2 = m^2\, \|R R^\top - I_k\|_F^2 = m^2\, \|S^\top S\|_F^2,
    \end{equation}
    where the last equality is~\eqref{eq:RRt_minus_I_norm}. Taking expectations,
    \begin{equation}\label{eq:lproj_equal_lemma}
        \ell_{\text{proj}} = m^2\, \mathbb{E}\left[\|S^\top S\|_F^2\right].
    \end{equation}

    \paragraph{The Improvement.}
    Combining~\eqref{eq:lbasis_equal} and~\eqref{eq:lproj_equal_lemma},
    \begin{equation}
        \Delta \ell_{\text{proj}} = \ell_{\text{basis}} - \ell_{\text{proj}} = m^2\, \mathbb{E}\!\left[2\|S\|_F^2 - \|S^\top S\|_F^2\right].
    \end{equation}
    Since $R^\top R + S^\top S = I_k$ with $R^\top R \succeq 0$, we have $S^\top S \preceq I_k$, so all eigenvalues of $S^\top S$ lie in $[0,1]$. Letting $s_1,\ldots,s_k \in [0,1]$ denote these eigenvalues,
    \begin{equation}
        2\|S\|_F^2 - \|S^\top S\|_F^2 = \sum_{i=1}^k s_i(2 - s_i) \geq \sum_{i=1}^k s_i = \|S\|_F^2,
    \end{equation}
    and thus we have
    \begin{equation}
        \Delta \ell_{\text{proj}} \geq m^2\, \mathbb{E}\left[\|S\|_F^2\right] > 0.
    \end{equation}
    Therefore, projecting onto the true principal eigenspace strictly decreases the denoising loss which concludes the proof.
\end{proof}

\subsection{Proof of Theorem~\ref{thm: softmax reduce error}}\label{app: proof of softmax reduce error}
\begin{proof}

    Let $E(\alpha) = \sigma(\alpha\hat{A}) - P_U$ and $\lambda_{\max} = \max_i m_i = \|\hat{A}\|_2$.
    Then we have:
\begin{align}
    f(\hat{A};\alpha) - A &= (P_U + E)\,\hat{A}\,(P_U + E)^\top - A \notag\\
    &= \underbrace{(P_U\hat{A}P_U^\top - A)}_{\delta_0} \;+\; \underbrace{E\hat{A}P_U^\top + P_U\hat{A}E^\top + E\hat{A}E^\top}_{\delta_1}
\end{align}
where we used $P_U A P_U^\top = U_k U_k^\top (U_k \Lambda_k U_k^\top) U_k U_k^\top = U_k \Lambda_k U_k^\top = A$,
so that $P_U\hat{A}P_U^\top - A = P_U(\hat{A} - A)P_U^\top = \delta_0$.
Expanding the squared Frobenius norm:
\begin{equation}
    \|f(\hat{A};\alpha) - A\|_F^2 = \|\delta_0\|_F^2 + 2\langle \delta_0,\, \delta_1\rangle_F + \|\delta_1\|_F^2
\end{equation}
Taking expectations and noting $\mathbb{E}[\|\delta_0\|_F^2] = \ell_{\mathrm{proj}}$ gives~\eqref{eq:gap_exact}.
\begin{equation}\label{eq:gap_exact}
    \ell(\alpha) = \ell_{\mathrm{proj}} + 2\,\mathbb{E}\!\left[\langle \delta_0,\, \delta_1 \rangle_F\right] + \mathbb{E}\!\left[\|\delta_1\|_F^2\right]
\end{equation}
Applying Cauchy--Schwarz to the cross term:
\begin{equation}
    |\mathbb{E}[\langle \delta_0, \delta_1\rangle_F]| \leq \mathbb{E}[\|\delta_0\|_F\|\delta_1\|_F] \leq \sqrt{\mathbb{E}[\|\delta_0\|_F^2]}\;\sqrt{\mathbb{E}[\|\delta_1\|_F^2]} = \sqrt{\ell_{\mathrm{proj}}}\;\sqrt{\mathbb{E}[\|\delta_1\|_F^2]}
\end{equation}
Thus we have:
\begin{equation}
    \ell(\alpha) \leq \ell_{\mathrm{proj}} + 2\sqrt{\ell_{\mathrm{proj}}}\;\sqrt{\mathbb{E}[\|\delta_1\|_F^2]} + \mathbb{E}\!\left[\|\delta_1\|_F^2\right]
\end{equation}
It remains to bound $\mathbb{E}[\|\delta_1\|_F^2]$ and show that:
\begin{equation}\label{eq:gap_bound}
    2\sqrt{\ell_{\mathrm{proj}}}\;\sqrt{\mathbb{E}[\|\delta_1\|_F^2]} + \mathbb{E}\!\left[\|\delta_1\|_F^2\right] \;<\; \ell_{\mathrm{basis}} - \ell_{\mathrm{proj}}
\end{equation}

\noindent \textbf{Bounding $\mathbb{E}[\|\delta_1\|_F^2]$.}

Since $\hat{A} = \tilde{U}_k \Lambda_k \tilde{U}_k^\top$ and $P_U = U_k U_k^\top$, denoting $\Omega = \tilde{U}_k^\top U_k \in \mathbb{R}^{k \times k}$ (which satisfies $\|\Omega\|_2 \leq 1$), every term in $\delta_1$ factors through $E\tilde{U}_k$:
\begin{align}
    E\hat{A}P_U^\top &= (E\tilde{U}_k)\,\Lambda_k\,\Omega\, U_k^\top, \label{eq:EAP}\\
    P_U\hat{A}E^\top &= U_k\,\Omega^\top\Lambda_k\,(E\tilde{U}_k)^\top, \label{eq:PAE}\\
    E\hat{A}E^\top &= (E\tilde{U}_k)\,\Lambda_k\,(E\tilde{U}_k)^\top. \label{eq:EAE}
\end{align}
Therefore, by sub-multiplicativity:
\begin{equation}\label{eq:delta1_tight}
    \|\delta_1\|_F \leq \lambda_{\max}\!\left(2\|E\tilde{U}_k\|_F + \|E\tilde{U}_k\|_F^2\right)
\end{equation}
Now we turn to bounding $\|E\tilde{U}_k\|_F$.

\noindent \textbf{Bounding $\|E\tilde{U}_k\|_F$.} Write $\tilde{U}_k = U_k + \Delta U_k$. Then we have:
\begin{equation}
    \|E\tilde{U}_k\|_F = \|E(U_k + \Delta U_k)\|_F = \|E U_k\|_F + \|E \Delta U_k\|_F
\end{equation}

The Following Lemma allows us to bound $\|E \Delta U_k\|_F$:
\begin{lemma}[Bound on $EU_k$]\label{lem:EVk}
    Define $\eta = \|\hat{A} - A\|_\infty := \max_{i,j}|\hat{A}_{ij} - A_{ij}|$ and $\gamma = 1 - 2\eta$.
    Assume $\gamma > 0$ (i.e., the noise is small enough that in-block and cross-block entries of $\hat{A}$ are separated). Then
    \begin{equation}\label{eq:EVk_bound}
        \|E U_k\|_F^2 \;\leq\; e^{-2\alpha\gamma} \sum_{\ell=1}^{k}\left(\frac{(n - m_\ell)^2}{m_\ell^2} + m_\ell \!\sum_{\ell' \neq \ell} \frac{1}{m_{\ell'}}\right).
    \end{equation}
    In the equal block case ($m_\ell = m$, $n = km$), this simplifies to
    \begin{equation}\label{eq:EVk_equal}
        \|E U_k\|_F \;\leq\; k\sqrt{k-1}\; e^{-\alpha\gamma}.
    \end{equation}
\end{lemma}

Using Lemma~\ref{lem:EVk} we can write in the equal block case:
\begin{equation}\label{eq:EVtilde_bound}
    \|E\tilde{U}_k\|_F \leq k\sqrt{k-1}\; e^{-\alpha\gamma} + \|E\Delta U_k\|_F
\end{equation}

Now we will write the LHS and RHS of the inequality~\eqref{eq:gap_bound} up to leading order in $\|S\|_F$.

\vspace{0.3cm}
\noindent \textbf{RHS of~\eqref{eq:gap_bound} to leading order.} From Lemma~\ref{lem: projection helps} we recall,
\begin{equation}
    \ell_{\text{basis}} = 2m^2 \mathbb{E}\left[\|S\|_F^2\right], \qquad
    \ell_{\text{proj}} = m^2\, \mathbb{E}\left[\|S^\top S\|_F^2\right].
\end{equation}
Since the eigenvalues of $S^\top S$ lie in $[0,1]$, $\|S^\top S\|_F^2 \leq \|S\|_F^4$. Hence, to leading order in $\|S\|_F$,
\begin{equation}\label{eq:RHS_leading}
    \ell_{\text{basis}} - \ell_{\text{proj}} = 2m^2\, \mathbb{E}\left[\|S\|_F^2\right] + O\!\left(\mathbb{E}\|S\|_F^4\right).
\end{equation}

\vspace{0.3cm}
\noindent \textbf{LHS of~\eqref{eq:gap_bound} to leading order.} We first bound $\|E\tilde{U}_k\|_F$. Decompose $E = E_1 + E_2$ with
\begin{equation}
    E_1 = \sigma(\alpha A) - P_U, \qquad E_2 = \sigma(\alpha\hat{A}) - \sigma(\alpha A).
\end{equation}
$E_1$ is deterministic (independent of noise) and captures the finite-$\alpha$ bias of the softmax; a direct calculation shows $\|E_1\|_F = O(\sqrt{n}\, e^{-\alpha})$, decaying exponentially in $\alpha$. The noise term $E_2$ is controlled by the following Lemma (See Section~\ref{app: proof of E2_bound} for proof).

\begin{lemma}[Linearized softmax bound]\label{lem:E2_bound}
    In the equal-block case, to leading order in $\|S\|_F$,
    \begin{equation}\label{eq:E2_bound}
        \|E_2\|_F \;\leq\; \sqrt{2}\,\alpha\,\|S\|_F + O(\|S\|_F^2).
    \end{equation}
\end{lemma}

Since $\|E_1\|_F$ decays exponentially in $\alpha$ (and is dominated by $\|E_2\|_F$ for $\alpha\gamma \gg \log k$ and $\|S\|_F > 0$),
\begin{equation}
    \|E\|_F \leq \|E_1\|_F + \|E_2\|_F = \sqrt{2}\,\alpha\,\|S\|_F + O(\|S\|_F^2).
\end{equation}
Furthermore, $\|\Delta U_k\|_F = \|S\|_F + O(\|S\|_F^2)$ to leading order (the within-eigenspace component $\|R - I_k\|_F$ is $O(\|S\|_F^2)$). Substituting into~\eqref{eq:EVtilde_bound},
\begin{equation}
    \|E\tilde{U}_k\|_F \leq k\sqrt{k-1}\, e^{-\alpha\gamma} + \sqrt{2}\,\alpha\|S\|_F^2 + o(\|S\|_F^2).
\end{equation}
Assume $\alpha\gamma \gg \log k$ so the exponential term is negligible; then
\begin{equation}\label{eq:EVtilde_leading}
    \|E\tilde{U}_k\|_F = \sqrt{2}\,\alpha\|S\|_F^2 + o(\|S\|_F^2),
\end{equation}
and the quadratic term $\|E\tilde{U}_k\|_F^2 = O(\|S\|_F^4)$ is higher order. Plugging into~\eqref{eq:delta1_tight} with $\lambda_{\max} = m$,
\begin{equation}
    \|\delta_1\|_F \leq m\!\left(2\|E\tilde{U}_k\|_F + \|E\tilde{U}_k\|_F^2\right) = 2\sqrt{2}\,m\alpha\|S\|_F^2 + o(\|S\|_F^2),
\end{equation}
so
\begin{equation}\label{eq:delta1_sq_leading}
    \mathbb{E}\|\delta_1\|_F^2 \leq 8 m^2 \alpha^2\, \mathbb{E}\|S\|_F^4 + o\!\left(\mathbb{E}\|S\|_F^4\right).
\end{equation}
Since $\sqrt{\ell_{\text{proj}}} \leq m\sqrt{\mathbb{E}\|S\|_F^4}$ (using $\|S^\top S\|_F^2 \leq \|S\|_F^4$), the cross term satisfies
\begin{equation}
    2\sqrt{\ell_{\text{proj}}}\sqrt{\mathbb{E}\|\delta_1\|_F^2} \leq 4\sqrt{2}\, m^2 \alpha\, \mathbb{E}\|S\|_F^4 + o\!\left(\mathbb{E}\|S\|_F^4\right).
\end{equation}
Combining,
\begin{equation}\label{eq:LHS_leading}
    \text{LHS of~\eqref{eq:gap_bound}} \;\leq\; m^2\alpha\!\left(4\sqrt{2} + 8\alpha\right)\, \mathbb{E}\|S\|_F^4 + o\!\left(\mathbb{E}\|S\|_F^4\right).
\end{equation}

\vspace{0.3cm}
\noindent \textbf{Resulting condition.} Combining~\eqref{eq:RHS_leading} and~\eqref{eq:LHS_leading}, the bound~\eqref{eq:gap_bound} holds to leading order whenever
\begin{equation}\label{eq:leading_condition}
    \alpha\!\left(4\sqrt{2} + 8\alpha\right) \frac{\mathbb{E}\|S\|_F^4}{\mathbb{E}\left[\|S\|_F^2\right]} \;<\; 2.
\end{equation}

Denote $\beta := 2\dfrac{\mathbb{E}\left[\|S\|_F^2\right]}{\mathbb{E}\|S\|_F^4}$. Then the condition~\eqref{eq:leading_condition} becomes
\begin{equation}\label{eq:leading_condition_beta}
    \alpha\!\left(4\sqrt{2} + 8\alpha\right)\;<\; \beta.
\end{equation}
Solving the quadratic $8\alpha^2 + 4\sqrt{2}\,\alpha - \beta = 0$, the condition holds for $\alpha \in (0, \alpha_+)$ where
\begin{equation}\label{eq:alpha_plus}
    \alpha_+ \;=\; \frac{\sqrt{2}}{4}\!\left(-1 + \sqrt{1 + \beta}\right).
\end{equation}

\paragraph{Compatibility of the two conditions on $\alpha$.}
The proof so far rests on two conditions on the parameter $\alpha$:
\begin{itemize}
    \item[\textbf{(A)}] $\alpha \in (0, \alpha_+)$, so that the LHS leading-order bound~\eqref{eq:LHS_leading} is dominated by the RHS leading-order bound~\eqref{eq:RHS_leading} (this is the bound~\eqref{eq:leading_condition_beta}).
    \item[\textbf{(B)}] $\alpha\gamma \gg \log k$, so that the exponential term $k\sqrt{k-1}\,e^{-\alpha\gamma}$ in~\eqref{eq:EVtilde_bound} is negligible compared to the noise term.
\end{itemize}
These conditions pull in opposite directions: (A) caps $\alpha$ from above, while (B) requires $\alpha$ to grow at least like $\log k$. To show that they can be met simultaneously, we make the following assumption on the noise.

\begin{assumption}\label{ass:noise_small}
    There exists $s$ such that $\|S\|_F \leq s$ almost surely, with
    \begin{equation}\label{eq:noise_small_assumption}
        s \;\ll\; \min\!\left(\frac{1}{m},\; \frac{1}{\log k}\right).
    \end{equation}
\end{assumption}

Note that $\eta = \|\hat{A}-A\|_\infty$ is not an independent quantity: by the entry-wise bound $\|M\|_\infty \leq \|M\|_F$ together with $\|\hat{A}-A\|_F = m\sqrt{2}\|S\|_F$ from~\eqref{eq:lbasis_equal},
\begin{equation}\label{eq:eta_bound}
    \eta \;\leq\; \|\hat{A}-A\|_F \;=\; m\sqrt{2}\,\|S\|_F \;\leq\; m\sqrt{2}\,s.
\end{equation}
Therefore Assumption~\ref{ass:noise_small} caps both $\|S\|_F$ \emph{and} $\eta$. The two bounds in~\eqref{eq:noise_small_assumption} will turn out to control $\gamma$ (via $\eta \lesssim ms$) and the exponential term (via $\alpha\gamma \gtrsim 1/s$), making (A) and (B) compatible.

\smallskip
\noindent\textbf{Lower bound on $\alpha_+$.}
Since $\|S\|_F^2 \leq s^2$ a.s., we have the pointwise bound $\|S\|_F^4 \leq s^2\,\|S\|_F^2$. Taking expectations,
\begin{equation}
    \mathbb{E}\|S\|_F^4 \;\leq\; s^2\, \mathbb{E}\|S\|_F^2
    \quad\Longrightarrow\quad
    \beta \;=\; 2\,\frac{\mathbb{E}\|S\|_F^2}{\mathbb{E}\|S\|_F^4} \;\geq\; \frac{2}{s^2}.
\end{equation}
Since $=\alpha_+ = \tfrac{\sqrt{2}}{4}(\sqrt{1+\beta}-1)$ is monotone increasing in $\beta$ we can write:
\begin{equation}\label{eq:alpha_plus_lower}
    \alpha_+ \;\geq\; \tfrac{\sqrt{2}}{4}\!\left(\sqrt{1+2/s^2}-1\right) \;\geq\; \tfrac{\sqrt{2}}{4}\!\left(\tfrac{\sqrt{2}}{s}-1\right) \;=\; \tfrac{1}{2s} - \tfrac{\sqrt{2}}{4},
\end{equation}
where we used $\sqrt{1+2/s^2} \geq \sqrt{2}/s$. Thus the cap $\alpha_+$ is at least of order $1/s$, which is large under Assumption~\ref{ass:noise_small}.

\smallskip
\noindent\textbf{Lower bound on $\gamma$.}
Combining $\gamma = 1 - 2\eta$ with~\eqref{eq:eta_bound},
\begin{equation}\label{eq:gamma_lower}
    \gamma \;\geq\; 1 - 2m\sqrt{2}\,s.
\end{equation}
The condition $ms \ll 1$ from Assumption~\ref{ass:noise_small} ensures $2m\sqrt{2}\,s \leq 1/2$ for $s$ small enough, hence $\gamma \geq 1/2$.

\smallskip
\noindent\textbf{Choosing $\alpha^\star$ to satisfy both (A) and (B).}
Set $\alpha^\star := \dfrac{1}{4s}$. We verify each condition:
\begin{itemize}
    \item \textbf{Condition (A):} $\alpha^\star < \alpha_+$. By~\eqref{eq:alpha_plus_lower} this reduces to $\tfrac{1}{4s} < \tfrac{1}{2s} - \tfrac{\sqrt{2}}{4}$, i.e., $s < \tfrac{1}{\sqrt{2}}$, which is automatic since $s \ll 1/\log k$.
    \item \textbf{Condition (B):} $\alpha^\star\gamma \gg \log k$. By~\eqref{eq:gamma_lower}, once $ms \ll 1$ so that $\gamma \geq 1/2$ we have,
    \begin{equation}
        \alpha^\star \gamma \;\geq\; \tfrac{1}{4s}\cdot\tfrac{1}{2} \;=\; \tfrac{1}{8s} \;\gg\; \log k,
    \end{equation}
    where the last step uses $s \ll 1/\log k$.
\end{itemize}
This is precisely why both parts of~\eqref{eq:noise_small_assumption} are needed: $ms \ll 1$ guarantees $\gamma$ stays bounded away from $0$, and $s \log k \ll 1$ then makes $\alpha^\star\gamma$ large compared to $\log k$. Hence $\alpha^\star$ lies in $(0, \alpha_+)$ \emph{and} makes the exponential term in~\eqref{eq:EVtilde_bound} negligible, so all the leading-order bounds derived in the proof apply at $\alpha = \alpha^\star$.

\smallskip
\noindent\textbf{Conclusion.}
At $\alpha = \alpha^\star$, the condition~\eqref{eq:leading_condition_beta} holds strictly, which by~\eqref{eq:LHS_leading} and~\eqref{eq:RHS_leading} is equivalent to the gap bound~\eqref{eq:gap_bound}. Therefore
\begin{equation}
    \epsilon(\alpha^\star) \;=\; \ell_{\text{soft}}(\alpha^\star) - \ell_{\text{proj}} \;<\; \ell_{\text{basis}} - \ell_{\text{proj}} \;=\; \Delta\ell_{\text{proj}},
\end{equation}
proving the theorem.

\paragraph{Translation to the noise scale.}
We close by translating Assumption~\ref{ass:noise_small} into a condition on the adjacency-level noise $\mathcal{E} := \tilde{A} - A$. In the equal-block case the spectral gap of $A$ between the $k$-th principal eigenvalue and the trivial part of the spectrum equals $\Delta = m$. By the Davis--Kahan theorem (in $\sin\Theta$ form),
\begin{equation}\label{eq:davis_kahan}
    \|S\|_F \;\leq\; \frac{\|P_\perp\, \mathcal{E}\, V_k\|_F}{\Delta} \;\leq\; \frac{\|\mathcal{E}\|_F}{m}.
\end{equation}
Squaring and taking expectations,
\begin{equation}\label{eq:sigma2_from_E}
    \sigma^2 \;=\; \mathbb{E}\|S\|_F^2 \;\leq\; \frac{\mathbb{E}\|\mathcal{E}\|_F^2}{m^2}.
\end{equation}
Translating the two parts of~\eqref{eq:noise_small_assumption} (with $s^2$ scaling like $\mathbb{E}\|\mathcal{E}\|_F^2 := \sigma^2$, assuming concentration of $\|S\|_F^2$ around its mean):
\begin{itemize}
    \item $s \ll 1/m$ becomes $\sigma^2 \ll 1/m^2$, i.e., $\mathbb{E}\|\mathcal{E}\|_F^2 \ll 1$.
    \item $s \ll 1/\log k$ becomes $\sigma^2 \ll 1/\log^2 k$, i.e., $\mathbb{E}\|\mathcal{E}\|_F^2 \ll m^2/\log^2 k$.
\end{itemize}
Combining, Assumption~\ref{ass:noise_small} holds provided
\begin{equation}\label{eq:noise_scale_condition}
    \;\mathbb{E}\|\mathcal{E}\|_F^2 \;\ll\; \min\!\left(1,\; \frac{m^2}{\log^2 k}\right).\;
\end{equation}
which is the assumption stated in the theorem.

For independent edge flips with probability $p$, $\mathbb{E}[\mathcal{E}_{ij}^2] = p(1-p)$ for each entry, so $\mathbb{E}\|\mathcal{E}\|_F^2 = n(n-1)\,p(1-p) \approx n^2 p$. Substituting and using $n = km$,
\begin{equation}
    p \;\ll\; \min\!\left(\frac{1}{n^2},\; \frac{1}{k^2 \log^2 k}\right).
\end{equation}

\end{proof}

\subsubsection{Proof of Lemma~\ref{lem:EVk}}\label{app: proof of EVk}

\begin{proof}
We bound $(Eu_\ell)_i$ separately for $i$ in-block ($i \in B_\ell$) and cross-block ($i \in B_{\ell'},\, \ell' \neq \ell$). Write $S = \sigma(\alpha\hat{A}) = P_U + E$ and $Z_i = \sum_{j'=1}^n e^{\alpha\hat{A}_{ij'}}$ for the softmax normalizer.

\vspace{0.2cm}
\noindent\textbf{Case 1: $i \in B_\ell$.}
Since $\sum_j E_{ij} = 0$ and $(P_U)_{ij} = 0$ for $j \notin B_\ell$:
\begin{equation}
    \sum_{j \in B_\ell} E_{ij} = -\sum_{j \notin B_\ell} E_{ij} = -\sum_{j \notin B_\ell} S_{ij}.
\end{equation}
For $j \notin B_\ell$: $\hat{A}_{ij} \leq \eta$ (cross-block), so $e^{\alpha\hat{A}_{ij}} \leq e^{\alpha\eta}$.
For $j' \in B_\ell$: $\hat{A}_{ij'} \geq 1 - \eta$ (same-block), so $Z_i \geq m_\ell\, e^{\alpha(1-\eta)}$.
Therefore
\begin{equation}
    \left|\sum_{j \in B_\ell} E_{ij}\right| \leq \frac{(n - m_\ell)\,e^{\alpha\eta}}{m_\ell\, e^{\alpha(1-\eta)}} = \frac{n - m_\ell}{m_\ell}\,e^{-\alpha\gamma}
\end{equation}
and $|(Ev_\ell)_i| \leq \frac{n-m_\ell}{m_\ell^{3/2}}\,e^{-\alpha\gamma}$.

\vspace{0.2cm}
\noindent\textbf{Case 2: $i \in B_{\ell'},\; \ell' \neq \ell$.}
Now $(P_U)_{ij} = 0$ for all $j \in B_\ell$, so $\sum_{j \in B_\ell} E_{ij} = \sum_{j \in B_\ell} S_{ij}$.
For $j \in B_\ell$: $\hat{A}_{ij} \leq \eta$ (cross-block relative to $i$). For $j' \in B_{\ell'}$: $Z_i \geq m_{\ell'}\, e^{\alpha(1-\eta)}$. Thus
\begin{equation}
    \sum_{j \in B_\ell} S_{ij} \leq \frac{m_\ell\,e^{\alpha\eta}}{m_{\ell'}\,e^{\alpha(1-\eta)}} = \frac{m_\ell}{m_{\ell'}}\,e^{-\alpha\gamma}
\end{equation}
and $|(Ev_\ell)_i| \leq \frac{\sqrt{m_\ell}}{m_{\ell'}}\,e^{-\alpha\gamma}$.

\vspace{0.2cm}
\noindent\textbf{Combining.} We have
\begin{equation}
    \|Ev_\ell\|^2 = \sum_{i \in B_\ell}(Ev_\ell)_i^2 + \sum_{\ell' \neq \ell}\sum_{i \in B_{\ell'}}(Ev_\ell)_i^2
    \leq e^{-2\alpha\gamma}\!\left(\frac{(n-m_\ell)^2}{m_\ell^2} + m_\ell\!\sum_{\ell' \neq \ell}\frac{1}{m_{\ell'}}\right).
\end{equation}
Summing over $\ell = 1, \ldots, k$ gives~\eqref{eq:EVk_bound}. Setting $m_\ell = m$ and $n = km$ yields~\eqref{eq:EVk_equal}.
\end{proof}

\subsubsection{Proof of Lemma~\ref{lem:E2_bound}}\label{app: proof of E2_bound}
\begin{proof}
    Let $\sigma:\mathbb{R}^n \to \mathbb{R}^n$ denote the row-wise softmax with Jacobian $J(x) = \mathrm{diag}(\sigma(x)) - \sigma(x)\sigma(x)^\top$. For row $i \in B_\ell$ of $A$, $\sigma(\alpha A_{i,:}) = a\,\mathbf{1}_{B_\ell} + b\,\mathbf{1}_{B_\ell^c}$ with $a = 1/(m + (n-m)e^{-\alpha})$ and $b = a\,e^{-\alpha}$. The eigenvalues of $J(\alpha A_{i,:})$ are $0$ (on $\mathbf{1}$), $a$ (multiplicity $m-1$, on sum-zero perturbations within $B_\ell$), $b$ (multiplicity $n-m-1$, on sum-zero perturbations within $B_\ell^c$), and $nab$ (on $\mathrm{span}\{(n-m)\mathbf{1}_{B_\ell} - m\,\mathbf{1}_{B_\ell^c}\}$). For $\alpha \geq 0$, $a \geq b$ and $a \geq nab$, so
    \begin{equation}
        \|J(\alpha A_{i,:})\|_2 \;=\; a \;\leq\; \tfrac{1}{m}.
    \end{equation}
    Taylor-expanding $\sigma$ around $\alpha A_{i,:}$ row-wise (the second-order remainder is controlled by the uniformly bounded Hessian of softmax), squaring, and summing over rows,
    \begin{equation}
        \|E_2\|_F^2 \;\leq\; \alpha^2 \max_i \|J(\alpha A_{i,:})\|_2^2 \cdot \|\hat{A} - A\|_F^2 + O(\|\hat{A}-A\|_F^4) \;\leq\; \tfrac{\alpha^2}{m^2}\, \|\hat{A} - A\|_F^2 + O(\|S\|_F^4).
    \end{equation}
    The claim follows from $\|\hat{A} - A\|_F = m\sqrt{2}\|S\|_F$ in~\eqref{eq:lbasis_equal}.
\end{proof}

\subsection{Edge-flipping noise}
\label{app:edge_flipping}

For Digress noise, we consider the edge-flipping model where $\tilde A_{ij}=A_{ij}$ with prob. $1-\varepsilon$ and $\tilde A_{ij}=1-A_{ij}$ with prob. $\varepsilon$, so that $\tilde A = A + \mathcal{E}_1$ with $\mathcal{E}_1=\tilde A-A$. A direct computation gives $\mathbb{E}[\tilde A \mid A]=(1-2\varepsilon)A+\varepsilon \mathbf{1}\mathbf{1}^\top$, hence
\begin{equation}
     \tilde A = (1-2\varepsilon)A + \varepsilon \mathbf{1}\mathbf{1}^\top + \Xi,
\end{equation}
where $\Xi:=\tilde A-\mathbb{E}[\tilde A\mid A]$ has zero mean independent entries and is thus a Wigner-type matrix.

\clearpage
\section{Related Work}\label{app: related work}

\subsection{Optimization-based graph denoising} 

A large body of work formulates graph denoising as a model-based inverse problem in which the clean graph is obtained as the minimizer of a cost regularized by hand-crafted structural priors, with no learned mapping involved. From a network-statistics angle, \citep{feizi2013network} invert in closed form a convolutional model of direct and indirect edge effects, while \citep{wu2017generalized,josephs2021network} recover a clean adjacency under Type-I/II edge errors, low-rank parameter matrices, or unlabeled noisy samples, respectively. Within the graph signal processing (GSP) framework, the topology is denoised jointly with a process defined on it: \citep{segarra2017network} recover a graph shift operator from possibly corrupted spectral templates, and \citep{rey2021overparametrized, rey2023robust,tenorio2024blind} progressively extend this idea to robust filter identification and blind deconvolution under joint perturbations of inputs, filter, and topology. A third and particularly influential family exploits the low-rank, sparse, and smooth nature of real-world graphs, with Pro-GNN \citep{jin2020graph} as the canonical formulation and a number of follow-ups \citep{xu2021speedup,zhou2024graph,wang2024unsupervised} that improve scalability and avoid over-smoothing through truncated-SVD reparametrizations, framelet regularizers, or feature-driven matrix factorization. Related approaches obtain the clean graph from low-rank  motifs \citep{lyu2024learning} or via iterative Fourier thresholding of the adjacency \citep{frost2026denoising}; in the same spirit, \citep{tenorio2023robust} extends the Pro-GNN-style joint optimization by replacing its low-rank and sparsity priors with a GSP-style denoising regularizer that is alternated with gradient updates of the downstream GNN.

\subsection{Graph denoising via learned estimators} 

A complementary family of methods abandons the optimization viewpoint altogether and instead trains a parametric mapping from a noisy graph to a clean one, so that denoising amounts to a forward pass through a neural network rather than the solution of an instance-specific inverse problem. The first group of such methods uses the learned denoiser as a component of, or a preprocessing for, a downstream GNN. GADPN \citep{deng2026gadpn} performs the denoising step itself with a non-learned rank-adaptive SVD whose rank is chosen by Bayesian optimization, and feeds the cleaned adjacency to a standard GNN classifier. UGD \citep{yang2024you} couples a non-parametric high-order proximity score for edge pruning with a GCN-based graph autoencoder that reconstructs corrupted features, and alternates between the two. RS-GNN \citep{dai2022towards} parametrizes the denoiser as an MLP edge predictor that re-weights the adjacency from pairs of node features and is trained jointly with a downstream GCN, while SLAPS \citep{fatemi2021slaps} pairs an MLP-$k$NN graph generator with a separate GNN denoising autoencoder that recovers masked features through the learned structure. A second group places the parametric denoiser at the core of a denoising diffusion or flow-matching model, where a forward process gradually corrupts a clean graph and the learned denoiser is applied iteratively to invert it. Mask-GVAE \citep{li2021maskgvae} is the earliest of these, with a partition-based GCN variational autoencoder that performs blind denoising without paired supervision. GDSS \citep{jo2022score} parametrizes the score of a system of SDEs over node features and adjacency with a graph convolution / graph transformer backbone, while DiGress \citep{vignac2022digress} introduces a discrete diffusion in which a graph transformer is trained to predict the clean graph from a noisy input at every step; GruM \citep{jo2023graph}, DeFoG \citep{qin2024defog}, and DiPhon~\cite{rozada2026diphon} reuse essentially the same graph transformer architecture but replace the noise-prediction objective by, respectively, endpoint prediction in a diffusion mixture and discrete flow matching. NetDiff \citep{marcoccia2024netdiff} augments the DiGress-style graph transformer with cross-attentive modulation tokens to enforce global topological constraints in ad-hoc network generation, DiffSP \citep{luo2025diffsp} repurposes such graph diffusion models as adversarial structure purifiers on top of a GCN classifier, and SGDM \citep{trivedi2023leveraging} operates at the subgraph level and is in fact backbone-agnostic, instantiated on top of DiGress, GDSS or EDGE.

\subsection{Spectral graph transformers}
A line of research orthogonal to standard message-passing GNNs builds graph neural architectures by routing the spectrum of the graph Laplacian through Transformer-style modules. The earliest works inject the spectrum as input features. SpGAT \citep{he2020spectral} first constructs spectral node features from the graph Laplacian eigenvectors and eigenvalues with an MLP on top. These features are then used as input for a GCN, which performs node-wise classification by assigning a label to each surface vertex. 
SAN \citep{kreuzer2021rethinking} constructs spectral positional features from the graph Laplacian spectrum and adds them to the node features. These enriched features are then processed by a fully connected Transformer. 
The Graph Transformer of \citep{dwivedi2021generalizationtransformernetworksgraphs} is an earlier instance of this same template, using truncated Laplacian eigenvectors directly as a positional encoding for a sparse-attention Transformer. This idea of using spectral information as feature is also investigated in \citep{pengmei2024technical} by processing Laplacian eigenvalues with an MLP and usesing the result to initialize the Transformer’s global \texttt{[CLS]} token, with the resulting architecture still being a graph Transformer.
A second sub-line removes the eigenvector ambiguities (sign flips, basis rotations within an eigenspace) before they reach the Transformer. SignNet and BasisNet \citep{lim2022sign} are parametric encoders that produce sign- and basis-invariant embeddings of the eigenvectors, and are routinely plugged in as the positional-encoding module of a downstream graph Transformer. 
GIST \citep{rigotti2026gist} addresses the same problem from a different angle: it builds random spectral features by feeding Gaussian noise through a precomputed graph filter, and then uses the inner product matrix of the resulting features as a reweighting mechanism. Specifically, the architecture is a Transformer whose attention scores are reweighted by this spectral inner product matrix.
A third sub-line uses the Transformer not as a pose-processing module of the spectrum but as the parametrization of a graph filter itself. Specformer \citep{bo2023specformer} learns a graph filter by feeding the graph eigenvalues into a Transformer, which produces frequency responses, and then converts those responses back into a graph operator using the eigenvectors. This learned operator is then used as the propagation layer inside a GNN, so the rest of the model behaves like a graph neural network for prediction. 
PolyFormer \citep{ma2024polyformer} avoids eigendecomposition altogether by representing each node with a sequence of polynomial-basis tokens of the Laplacian, and then learning a node-wise filter through self-attention over those tokens. Eigenformer \citep{garg2024graph} pushes this idea further by dispensing with explicit positional encodings and instead designing a spectrum-aware attention mechanism whose biases depend directly on the Laplacian spectrum, so that the Transformer block itself acts as a learnable spectral operator.

\end{document}